\documentclass[a4paper]{article}[11pts]

\usepackage[english]{babel}
\usepackage[utf8x]{inputenc}
\usepackage[T1]{fontenc}

\normalsize

\usepackage[a4paper,top=1in,bottom=1in,left=1in,right=1in,marginparwidth=1in]{geometry}

\usepackage{setspace}
\usepackage{amsmath}
\usepackage{amssymb}
\usepackage{amsthm}
\usepackage{amsfonts, mathtools}
\usepackage{algorithm,algpseudocode}
\usepackage{bm}
\usepackage{bbm}
\usepackage{comment}
\usepackage{graphicx}
\usepackage{multirow, booktabs}
\usepackage{caption}
\usepackage{subcaption}
\usepackage[colorinlistoftodos]{todonotes}
\usepackage{url}

\newcommand{\E}{\mathbb{E}}

\newcommand{\prob}{\mathbb{P}}

\DeclareMathOperator*{\argmin}{arg\,min}

\newtheorem{lemma}{Lemma}

\newtheorem{theorem}{Theorem}
\newtheorem{corollary}{Corollary}
\newtheorem{definition}{Definition}

\usepackage{natbib}

\newcommand{\hz}[1]{{\color{blue} [HZ: {#1}]}}

\newcommand{\ktt}[1]{{\color{red} [KTT: {#1}]}}

\title{A Neural Network Based Choice Model for Assortment Optimization}
\author{}
\date{}

\author{Hanzhao Wang, Zhongze Cai, Xiaocheng Li, Kalyan Talluri}
\date{\small
Imperial College Business School, Imperial College London\\
(h.wang19, z.cai22, xiaocheng.li, kalyan.talluri)@imperial.ac.uk}

\begin{document}
\maketitle

\onehalfspacing

\begin{abstract}
Discrete-choice models are used in economics, marketing and revenue management to predict customer purchase probabilities, say as a function of prices and other features of the offered assortment. While they have been shown to be expressive, capturing customer heterogeneity and behaviour, they are also hard to estimate, often based on many unobservables like utilities; and moreover, they still fail to capture many salient features of customer behaviour. A natural question then, given their success in other contexts, is if neural   networks can eliminate the necessity of carefully building a context-dependent customer behaviour model and hand-coding and tuning the estimation.  It is unclear however how one would incorporate assortment effects into such a neural network, and also how one would optimize the assortment with such a black-box generative model of choice probabilities. In this paper we investigate first whether a single neural network architecture can predict purchase probabilities for datasets from various contexts and generated under various models and assumptions. Next, we develop an assortment optimization formulation that is solvable by off-the-shelf integer programming solvers.  We compare against a variety of benchmark discrete-choice models on simulated as well as real-world datasets, developing training tricks along the way to make the neural network prediction and subsequent optimization robust and comparable in performance to the alternates.
\end{abstract}

\section{Introduction}
What goes on in the consumer's mind during the purchase process? How, and why, do they decide to purchase? What features of the product or environment encourage them to purchase?  Given the heterogeneity and idiosyncratic thought processes of customers, these questions are unlikely to be answered definitely.  Nevertheless these are fundamental questions of interest to marketers, behavioural psychologists, economists and operations researchers.   Hence, for operational and algorithmic purposes, a number of stylized or reduced-form models have been proposed in the literature to explain individual-level decision outcomes as well as collective aggregate purchase behavior.

Discrete-choice models are one such class of models, widely used in marketing, economics, transportation and revenue management, especially suitable for situations where customers choose at most one of several alternatives. They are typically parsimonious models, based on micro-economic utility maximization principles in a stochastic framework, and expressive enough to incorporate product, customer and environment features, and tractable enough for estimation and incorporation into optimization models.  The classical Multinomial Logit choice model (MNL) is an early and prominent example that, along with mixed-logit, nested-logit and probit, has found widespread application in many fields, both in theoretical models as well as in empirical studies.  New alternatives to MNL, that are more flexible or more expressive, such as Markov chain choice \citep{blanchet2016markov} and Exponomial \citep{alptekinouglu2016exponomial}, have been proposed recently.

One could potentially design more elaborate choice models, folding in many known aspects of the purchase behavior. However, this line of development has some natural limitations as, one, the models have to be specialized and tuned to a particular industry or even to a firm's purchase funnel; two, would require developing and hand-coding its estimation which may be difficult both in theory and practice; and finally, using the model in an assortment optimization method may be computationally infeasible.

So, given the fundamental indeterminacy in how and why a customer chooses to purchase, our inability to observe all the data required by some of these models, and as the process itself is opaque and liable to change by industry and context, a reasonable research proposition is to search for a robust universal model that requires minimal tuning or expert modeling knowledge and works as well as behavioural models across all contexts and industries.  Naturally, machine learning methods, specifically neural networks, come to mind.   They have been applied with great success for many prediction tasks in industry precisely because they do not require fine-grained modeling specific to a situation and industry, are robust, and come with standardized training algorithms that work off-the-shelf.  With enough data they also have proved to be very good performers in many areas, sometimes proving to be superior even to hand-crafted models.

In this paper we develop a neural network based choice model suitable for assortment optimization.  The firm has to offer an assortment of products, potentially customized to each individual customer, and the customer either chooses one of the offered products or decides not to purchase.   This problem arises in a wide variety of industries where customers choose at most one item, such as hotel rooms, airline seats, or expensive one-off purchases in a single category such as automobiles, laptops or mobile phones.  (We leave extensions to multi-item purchases for future research.)

The challenges are the following:   First, the data is not at a scale typically required for neural networks training; a hotel for instance has only a few months of seasonal data it can use for predictions.  Second, the neural network predictions have to be used subsequently in an assortment optimization problem---essentially revenue-optimal subset-selection---and a neural network output promises no structure to make this optimization efficient.   Consequently we want a network as compact as possible so the optimization---in an integer programming formulation that we develop---is tractable for a reasonably large number of products.

Our contributions in this paper are the following:
\begin{enumerate}
\item We tackle the above-mentioned challenges and develop a neural network based choice model under two settings: (i) a feature-free setting where no additional features on products or customers are used in prediction and (ii) a feature-based setting that can utilize both  product and customer features.
\item We formulate an integer-program over the trained neural network to find an assortment that maximizes expected revenue.
\item We perform extensive numerical simulations that compare the predictive and optimization performance under the following scenarios:
    \begin{enumerate}
    \item  Synthetic simulations  where the ground-truth is generated by a panoply of models in the choice literature---MNL, Markov chain, non-parametric, mixed-logit---choice models and cross-validate them, both on raw prediction of choice probabilities, as well as the quality of the optimized assortment.
    \item Situations where the customer behavior does not follow the models' script or rational economic behaviour, but is documented experimentally and empirically  in the literature.
    \item Real-world assortment and sales transactions data from the following important industries: physical fashion-retail, airline, hotel, and transportation, and compare the predictive performances of the popular discrete-choice models vs. our neural network.
    \end{enumerate}
\item We device a meta-learning trick that allows us to obtain good robust prediction performance from a compact neural network even when trained on limited transactional purchase data. Moreover, we show a well-trained network can be warm-started when adding new products to the pool.
\end{enumerate}

Our numerical results suggest that the neural network gives comparable performance to the best choice model in most cases, both in simulations as well as on real data.  Some of the parametric models do perform better in simulations, unsurprisingly, when the ground-truth data-generation coincides with the model.   We find a similar pattern for real-world data, where on a hold-out sample our neural network gives robust predictions across the different industries. Moreover, we find that a one-layer neural choice model can capture what may be considered ``irrational" purchase behavior---documented to occur in experiments and empirical studies---that cannot be explained by random utility models.  In summary,  shallow neural networks enjoy both prediction power and computational tractability for assortment optimization.

\section{Literature Review}

Probabilistic choice models originated in economics based on a stochastic utility model of the consumer.  The oldest and a still widely used discrete-choice model is the MNL, with alternates and generalizations proposed to overcome its known limitations, such as the nested-logit and mixed-logit.  We refer the reader to the seminal paper~\citet{mcfadden2000mixed} and the book \citet{Ben-Akiva85} for back-ground on discrete-choice models and their economic motivation and estimation from data.

We review the literature in Computer Science and Operations Research next on neural network modeling of discrete-choices and assortment optimization.

\subsection{Neural Network Models in Utility Modeling}
In the Computer Science community, \cite{bentz2000neural} use a neural network to capture the non-linear effects of features on the latent utility. A few recent works
\cite{wang2020deep,han2020neural,sifringer2020enhancing, wong2021reslogit,aouad2022representing,arkoudi2023combining} study different application contexts and explore different neural network architectures for the feature-utility mapping. In addition to the product and customer features, \cite{gabel2022product} also encode history purchase information. \cite{chen2021estimating,chen2022decision} take a different route and consider the random forest model as the mapping function. \cite{van2022choice} discuss the potential avenues for further integrating machine learning into choice modeling with a comprehensive review of the existing literature.

We term this line of works as deep learning based utility modeling instead of choice modeling, because these works mainly focus on predicting the unobserved product utilities with the observed features and applying an MNL-style operator for predicting purchase probabilities (except \cite{aouad2022representing}, who use a mixture of a set of estimated utilities to mimic the mixed-logit model). What is missing is the \textit{assortment effect} on utilities---the interactions between the products within (and possibly outside) the set of offered products. This has two negative implications: First, some of the neural network models require all the training samples to have a fixed-size assortment. This requirement is simply not met in many application contexts such as hotels, airlines or fashion retail where the assortments change frequently. Second, it fails to capture the effect that the assortment has on the product utilities. Intuitively, a product's utility is not only determined by product and customer features, but also affected by the assortment offered to the customer which is a distinguishing feature of our neural network.

Another common drawback of the existing deep learning approaches is that they all require the availability of product or customer features. When there are no features available, most of the models degenerate to a uniform choice probability. In contrast, our neural network based choice model is designed to handle the feature-free setting. Importantly, we make a distinction between the feature-free model and the feature-based model not by the availability of the features, but by whether the downstream application is interested in a population-level or personalized choice model for its assortment optimization.

\subsection{Assortment Optimization}
Assortment optimization is an important problem in operations management, used in the airline, hotel, and fashion retail  e-commerce. For more on the history and applications of assortment optimization,  we refer readers to the books \citet{talluri2004theory,gallego2019revenue}. Here we briefly review the complexity of the problem under some of the aforementioned choice models:

\noindent\textbf{Multinomial Logit Model (MNL):} Under the MNL model \cite{talluri2004revenue} show that the optimal assortment is a revenue-ordered nested assortment and thus the optimal policy is to order the items by their revenues to form a nested set of assortments and choose the assortment with the largest expected revenue in that nested set.  \cite{rusmevichientong2012robust} show such revenue-ordered policy (RO) is robust against the uncertainty in MNL parameters and solvable under a cardinality constraint.

\noindent\textbf{Markov Chain Choice Model (MCCM):} \cite{blanchet2016markov} give a polynomial-time algorithm for finding the optimal assortment under this model by repeatedly using a Bellman operator until convergence. \cite{feldman2017revenue} further prove that the optimal assortment can be found by solving a linear programming.  With constraints, \cite{desir2020constrained} prove that this assortment optimization problem is NP-hard and introduce algorithms with a provable worst-case approximation ratio.

\noindent\textbf{Mixed-Multinomial Logit model (MMNL):} \cite{bront2009column} show that even without constraints, the assortment optimization problem under MMNL is NP-hard and then provide a column-generation method to solve the problem.  \cite{rusmevichientong2014assortment} show that a revenue-ordered policy can obtain the optimal assortment in MMNL when the choice model has some special structures. \cite{mendez2014branch} propose a branch-and-cut algorithm for solving the optimization problem.

\noindent\textbf{A General Choice Model:} When the revenue function underlying a choice model is a black-box function, algorithms for finding the optimal assortment cannot be tailored by utilizing structure. \cite{jagabathula2014assortment} purposes the ADXOpt algorithm which repeatedly finds the best action in adding (A), deleting (D), and exchanging (X) based on the current assortment with provable approximation ratio under some conditions. \cite{udwani2021submodular} defines submodular order functions and purposes several heuristics for assortment optimization under various types of constraints. With mild conditions, the heuristics have bounded approximation ratios.
%
%
%

In this paper we show that in the context of assortment optimization, (1) a shallow network is enough to recover existing choice models and (2) good performance and robustness is achievable with a standard neural network and integer programming when dealing with different choice models compared to customized heuristics.

\section{Problem Setup}
The firm has a set of $n$ products $\mathcal{N}=\{1,2,...,n\}$ that can be offered to customers.   An assortment $\mathcal{S}\subseteq \mathcal{N}$ is a subset of $\mathcal{N}$ that the firm decides to present to the customer.   In practice, the assortment can be determined by the product availability or the seller may decide the assortment to maximize profits by limiting the choice of the customer. The choice model gives a prediction of the probability of the customer choosing product $i$ conditional on the assortment $\mathcal{S}$:
$$\prob(i|\mathcal{S}) \text{ for all } i\in\mathcal{N} \text{ and } \mathcal{S}\subseteq \mathcal{N}.$$
In particular, we assume that $\prob(i|\mathcal{S})=0$ for $i\notin \mathcal{S},$ i.e., the customer cannot choose a product not offered in the assortment. In this way, a choice model
\begin{equation}
    \mathcal{M} = \{\prob(i|\mathcal{S}): \mathcal{S}\subseteq \mathcal{N}\}
    \label{choice_model}
\end{equation}
dictates $2^n-1$ probability distributions, each of which corresponds to one possible nonempty assortment. In all applicable business contexts, there is usually a ``no-purchase'' option where the customer chooses not to purchase any of the items from the offered assortment. The no-purchase option can be captured by one product (say, indexed as $n$) in $\mathcal{N}$ and we can always have $ n \in \mathcal{S}.$

A tabular parameterization of the choice model $\mathcal{M}$ generally requires $\Omega(n\cdot 2^n)$ parameters. Additional structures are therefore imposed to facilitate efficient estimation and revenue optimization.

\subsection{Random Utility Models (RUM) and Beyond}
\label{subsec:RUM}
The \textit{random utility model} (RUM) is the dominant framework for choice models in practice. It assigns a random utility for each product and models the consumer's choice behavior according to the principle of utility maximization. Specifically, the utility of the $i$-th product
$$U_i=u_i+\epsilon_i \text{ for } i\in\mathcal{N}.$$
Here $u_i$ is deterministic and represents the mean utility (to the customer) of purchasing the $i$-th product among the population. The random quantities, $\epsilon_i$'s are assumed to have a zero mean and can be correlated among products, to model  unobservable factors as well as heterogeneity across products.  We note that customer heterogeneity is also modeled by assuming the random part is specific to product as well as the customer.
Under the RUM,
\begin{equation}
\prob(i|\mathcal{S}) \coloneqq \prob\left(U_{i} = \max_{j\in\mathcal{S}} U_{j}\right) \quad \text{for} \ i\in \mathcal{S}.
\label{eqn:RUM}
\end{equation}
For simplicity, we assume ties are broken randomly. Different distributions of $\epsilon_i$'s specialize in this general RUM framework into different choice models as in the following.

\noindent\textbf{Multinomial Logit Model (MNL):} The MNL choice model assumes $\epsilon_i$'s are i.i.d. from a mean-zero Gumbel distribution with scale parameter 1. Then the choice probability  for $i\in \mathcal{S}$  has a nice closed-form solution:
\begin{equation}
    \prob(i|\mathcal{S}) \coloneqq \frac{\exp(u_i)}{\sum_{j\in\mathcal{S}}\exp(u_{j})}.
    \label{eqn:MNL}
\end{equation}
When there is a feature vector $\bm{f}_i$ associated with each product, one can also represent the deterministic utility $u_i=\bm{\theta}^\top \bm{f}_i$ and obtain the linear-in-attributes MNL model. Several recent works model $u_i$ with a neural network function of the feature $\bm{f}_i$ \citep{bentz2000neural,sifringer2020enhancing,wang2020deep,arkoudi2023combining}.

Although MNL model is simple to  estimate and optimize for an assortment, the model's assumptions make it restricted to the \textit{independence of irrelevant alternatives} (IIA)  property: for any $i,j\in \mathcal{S}\subset\mathcal{S'}$,
$$\frac{\mathbb{P}(i|\mathcal{S})}{\mathbb{P}(j|\mathcal{S'})}=\frac{\mathbb{P}(i|\mathcal{S})}{\mathbb{P}(j|\mathcal{S'})}.$$
This means the ratio of two products' choice probabilities is independent of the other products in the assortment.  We illustrate the effect of this for prediction and contrast it with the other methods.

Table~\ref{tab:IIA_example} is an example to illustrate the limitation of the IIA property: In case I there is one product $A$ with a no-purchase option. In Case II, when we add another product $A'$ identical to the product $A$, we assume the customer who would purchase product $A$ in Case I will have equal probability to purchase $A$ or $A'$ in Case II. From the predicted probabilities, we can see the estimated MNL model cannot capture this purchase behavior due to the IIA property.

\begin{table}[ht!]
    \centering
    \begin{tabular}{cc|c|ccc}
    \toprule
        &\multirow{2}{*}{Product} & \multirow{2}{*}{True Prob.} & \multicolumn{3}{c}{Pred. Prob.} \\
        & &  &NN&MNL-MLE&MCCM-EM\\
        \midrule
        \multirow{2}{*}{Case I}
        &No-purchase  & .40 &.42&.46 &.40   \\
        &A   &.60 &.58 &.54 &.60 \\
        \midrule
        \multirow{3}{*}{Case II}
        &No-purchase  & .40&.42 &.32 &.38 \\
        &A   &  .30&.28 &.38 &.32 \\
        &A'  & .30&.30 &.30 &.30 \\

    \bottomrule

    \end{tabular}
    \caption{Example illustrating the effects of the IIA property of MNL. Column True Prob. is the true choice probabilities in each case and Pred. Prob. is the predicted probabilities from each method:  NN is a one-layer Gated-Assort-Net neural choice model (that we introduce later), MNL-MLE is the MNL model calibrated via Maximum-Likelihood, while MCCM-EM implements the expectation-maximization algorithm \citep{csimcsek2018expectation} to fit a MCCM model.  All choice models are estimated based on 8000 training samples where each sample's assortment is randomly picked with equal probability. }
    \label{tab:IIA_example}
\end{table}

The following three RUM models (along with many others) were developed to mitigate this.

\noindent\textbf{Mixed-Multinomial Logit Model (MMNL):}
The MMNL choice model \citep{mcfadden2000mixed} is a generalization of MNL to multiple segments. It assumes there are several customer types and each type has its own deterministic utilities; the random terms are still assumed to be i.i.d. and Gumbel distruted. We denote $\mathcal{C}$ as the set of customer types and $\alpha_c$ as the probability (for a customer) of being $c$'s type. Then the choice probability for $i\in\mathcal{S}$ is:
\begin{equation*}
    \prob(i|\mathcal{S}) \coloneqq \sum_{c\in \mathcal{C}}\alpha_c \frac{\exp(u_{c,i})}{\sum_{j\in\mathcal{S}}\exp(u_{c,j})}.
\end{equation*}
For a continuous distribution of customer types, the summation is replaced by an integral.
\cite{mcfadden2000mixed} show any RUM can be represented by an MMNL choice model under mild conditions and thus MMNL also has a large model capacity. However, the estimation of an MMNL is challenging and many of the methods proposed in the literature fail to recover the latent types' parameters \citep{jagabathula2020conditional,hu2022learning}.

\noindent\textbf{Markov Chain Choice Model (MCCM):} The MCCM \citep{blanchet2016markov} defines a discrete-time Markov chain on the space of products $\mathcal{N}$ and models each customer's choice through the realization of a sample path. Specifically, it assumes a customer arrives to product $i$ with probability $\lambda_i$ (interpreted as the initial state of the Markov chain). Then if the product is in the assortment $\mathcal{S}$, the customer purchases it and leaves. Otherwise, the customer transitions from product $i$ to product $j$ with probability $\rho_{ij}$. Naturally, $\sum_{i\in\mathcal{N}} \lambda_i=1$ and $\sum_{j\in\mathcal{N}} \rho_{ij}=1$ for all $i\in\mathcal{N}.$ Let $X_t\in\mathcal{N}$ denote the product (state) that the customer visits at time $t$, and define the first hitting time $\tau=\{t\ge0: X_t\in\mathcal{S}\}.$ Then, under the MCCM,
$$\prob(i|\mathcal{S}) \coloneqq \prob(X_{\tau}=i).$$
The MCCM is a generalization of MNL, in the sense MNL arises for a specific form of the $\lambda$'s and $\rho$'s.  It has more parameters, hence has a larger model capacity to capture more complex choice behavior. \cite{berbeglia2016discrete} establishes MCCM as a RUM through a random walk argument.

\noindent\textbf{Non-parametric (NP) Choice Model:} The NP choice model or random orderings model has origins in classical microeconomic demand theory based on preference rankings \citep{block1959random} and recently reintroduced with a new estimation proposal in  \citep{farias2009data, farias2013nonparametric}. The model assumes that there exists a distribution $\lambda: \text{Perm}_\mathcal{N} \rightarrow[0,1]$ over the set $\text{Perm}_\mathcal{N}$ of all possible permutations of the products. There are $n!$ customer types and each customer type has a preference list of the products corresponding to one permutation in $\text{Perm}_\mathcal{N}$. Customers always purchase the most preferred product in the preference list. The choice probability is given by
$$\prob(i|\mathcal{S}) \coloneqq\sum_{\sigma \in \text{Perm}_i(\mathcal{S})} \lambda(\sigma) $$
where the set $$\text{Perm}_i(\mathcal{S})\coloneqq\{\sigma\in\text{Perm}_{\mathcal{N}}: \sigma(i)<\sigma(j)\text{ for all } j\neq i \in \mathcal{S}\}$$
contains all customer types/permutations under which product $i$ is the most preferable product in the assortment $\mathcal{S}.$

The NP choice model has exponentially many parameters and thus a larger model capacity than both MNL and MCCM. Indeed, \cite{block1959random} show that any RUM choice model can be represented by a NP model, and vice versa. Moreover, the estimation method proposed by \citep{farias2009data, farias2013nonparametric} needs to solve a linear programming with size $n!$. Although the linear programming enjoys good theoretical properties, it is computationally infeasible when $n$ is large and has to be further simplified or approximated by the sampling or representation method.

\noindent\textbf{Non RUM models:}\
RUM models have a \textit{regularity} property: if we add a new product to the assortment $\mathcal{S}$, then the purchase probability of all currently offered products will never increase.  However, as remarked by \cite{chen2022decision}, there is an increasing body of experimental evidence  which suggests that the aggregate choice behavior of customers is not always consistent with this regularity property, and thus often violates the premise of RUM. For instance, \cite{seshadri2019discovering} use the seminal experiment of \citep{tversky1985framing} and
\cite{chen2022decision} quote the behavioral experiment of \citep{ariely2008predictably} to motivate the development of choice models beyond the scope of RUMs. Here we use the examples from the above two behavioral experiments to show the limited capacity of RUM, while the neural choice model's potential to pick up such behaviour in the data.

\begin{itemize}
    \item Table~\ref{tab:Newspaper_example}  is an implementation of the  experiment in \citep{ariely2008predictably}: The price column shows the price of each option. In Case I, an assortment has two available products, while in Case II, a clearly inferior option of Print-Only is added. The addition of Print-Only twists customer utilities for the other two options, which is reflected by the change in the true choice probability.
    \item Table~\ref{tab:cycle_example} is an implementation of the gambles example in \citep{tversky1985framing}: The Win Prob. column shows the winning probability of each gamble and Payoff column indicates the winning payoff.  There exists a preference cycle over three gambles A, B, C with different winning probabilities and payoffs, where A is preferred to B, B is preferred to C based on the expected payoff. However, C is preferred to A since the winning probability of C is significantly larger than A.
\end{itemize}

\begin{table}[ht!]
    \centering
    \begin{tabular}{ccc|c|ccc}
    \toprule
        &\multirow{2}{*}{Option} & \multirow{2}{*}{Price} & \multirow{2}{*}{True Prob.} & \multicolumn{3}{c}{Pred. Prob.} \\
        & & & &NN&MNL-MLE&MCCM-EM\\
        \midrule
        \multirow{3}{*}{Case I}
        &No-purchase &-- & .14 &.14&.14 &.14  \\
        &Internet-Only  & \$59 &.57 &.58 &.43 &.43 \\
        &Print-\&-Internet & \$125 & .29&.28 &.43  &.43 \\
        \midrule
        \multirow{4}{*}{Case II}
        &No-purchase &-- & .14&.13 &.14 &.15 \\
        &Internet-Only  & \$59 &  .29&.31 &.43 &.43 \\
        &Print-\&-Internet & \$125 & .57&.55 &.43 &.42 \\
        &Print-Only & \$125 & .00  &.01 &.00 & .00\\
    \bottomrule

    \end{tabular}
    \caption{Example based on the behavioral experiment in \citep{ariely2008predictably}. The experiment setup is same as Table~\ref{tab:IIA_example}. }
    \label{tab:Newspaper_example}
\end{table}

\begin{table}[ht!]
    \centering
    \begin{tabular}{cccc|c|ccc}
    \toprule
        &\multirow{2}{*}{Gamble}& \multirow{2}{*}{Win Prob.} & \multirow{2}{*}{Payoff} & \multirow{2}{*}{True Prob.} & \multicolumn{3}{c}{Pred. Prob.} \\
        & & & & &NN&MNL-MLE&MCCM-EM\\
        \midrule
        \multirow{2}{*}{Case I}
        &A  & 1/4 &\$ $6$  &.75 &.76 &.42 &.62 \\
        &B & 1/3 &\$ $4$   & .25&.24 &.58  &.38 \\
        \midrule
        \multirow{2}{*}{Case II}
        &B  & 1/3 &\$ $4$  &.75 &.74 &.43 &.62 \\
        &C & 1/2 &\$ $2$   & .25&.26 &.57  &.38 \\
        \midrule
        \multirow{2}{*}{Case III}
        &A  & 1/4 &\$ $6$  &.20 &.20 &.36 &.24 \\
        &C & 1/2 &\$ $2$   & .80&.80 &.64  &.76 \\
    \bottomrule
    \end{tabular}
    \caption{Example based on the behavioral experiment in \citep{tversky1985framing}. The experiment setup is same as Table~\ref{tab:IIA_example}.  }
        \label{tab:cycle_example}
\end{table}

In sum, a parametric model that fits observed purchases in all situations is unlikely.
Certainly more elaborate models of consumer search and decisions exist, but their estimation, based on unobservable factors and limited firm-level data, and  subsequent optimization is daunting, and certainly will require a lot of development resources.  Hence, it is worth striving for a flexible large-capacity learning engine to predict the purchase probabilities directly without worrying about getting the precise consumer behaviour modeling right.

With this in mind, we develop neural network models that allow the utilities to be dependent on the assortment with or without including customer and product features. Hence our neural choice model allows violation of the regularity property, letting   the utilities change depending on the assortment that is offered.

\section{Neural Choice Model}
\label{sec:model}
\subsection{A Feature-free Neural Choice Model}
In this section, we introduce our feature-free neural choice model defined simply as a neural network function that maps the assortment input $\bm{S} \in \{0,1\}^n$ to the output $\bm{Y}\in [0,1]^n$:
$$\bm{Y}=g(\bm{S};\bm{\theta})$$
where $\bm{\theta}$ encapsulates all the model parameters. The input $\bm{S}=(s_1,...,s_n)$ is the binary encoding of an assortment $\mathcal{S}$ where $s_i=1$ if $i\in\mathcal{S}$ and $s_i=0$ if $i\notin\mathcal{S}$. The output $\bm{Y}=(Y_1,...,Y_n)$ is a vector supported on the probability simplex with $Y_i=\prob(i|\mathcal{S})$.

\begin{figure}[ht!]
    \centering
    \begin{subfigure}[b]{0.4 \textwidth}
        \centering
        \includegraphics[height=2.7in]{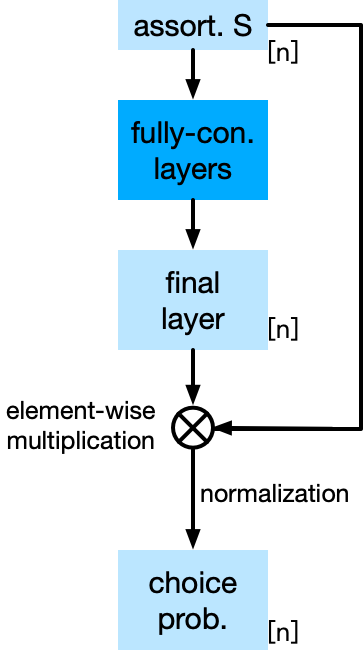}
        \caption{Gated-Assort-Net}
    \end{subfigure}%
    \begin{subfigure}[b]{0.4 \textwidth}
        \centering
        \includegraphics[height=2.7 in]{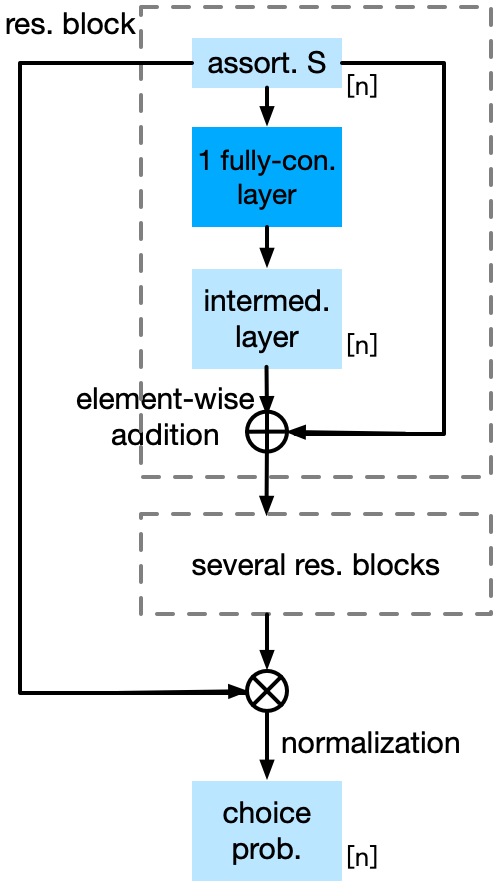}
        \caption{Res-Assort-Net}
    \end{subfigure}
    \caption{Neural choice models. $[n]$ denotes the dimension of the corresponding layer.}
    \label{fig:net_no_feat}
\end{figure}

We consider two neural network  architectures  illustrated in Figure~\ref{fig:net_no_feat}. For both architectures, we let $\bm{W}_l\in \mathbb{R}^{n_{l}\times n_{l-1}}$ and $\bm{b}_l\in \mathbb{R}^{n_{l}}$ are the weight matrix and bias vector respectively for layer $l$ with $n_l$ nodes in the $l$-th layer, and $n_1=n_L=n$.  We let $\bm{z}_{l} = (z_{l,1},...,z_{l,n_l}) \in \mathbb{R}^{n_l}$ be the output of the $l$-th layer and the input for the next layer. The parameters $L$ and $\{n_l,\ l=1,...,L\}$ need to be specified. (But we will show that $L=1$ and $n_l=n$ can perform well in both the real data and simulated data across different generative models.)

The final layer's output $\bm{z}_L=(z_{L,1},...,z_{L,n})$ are to set the choice probabilities given by a gated operator
$$
Y_i = \begin{cases}
\frac{\exp(z_{L,i})}{\sum_{i'\in\mathcal{S}} \exp(z_{L,i'})}, & i\in \mathcal{S},\\
0,& i\notin \mathcal{S}.
\end{cases}$$
While this may appear to be a standard MNL-like function, and hence suffer from the IIA property etc., it in fact does not; we note an important difference: The $\bm{z}_L$ is a function of $\bm{z}_0=\bm{S}$, i.e., the assortment effect is encoded into each $z_{L,i}$ and subsequently in $Y_i$ through the neural network layers.  (We should be hence be representing $\bm{z}_L$ as $\bm{z}_L(S)$ but for avoiding cumbersome notation, just write as $\bm{z}_L$. 

The differences in the architectures are as follows.

\paragraph{Gated-Assort-Net (GAsN)}\

The GAsN model is a fully-connected neural network designed specifically for the discrete-choice model. It adopts the classic structure commonly found in traditional neural networks, making it straightforward to implement. It takes the assortment vector $\bm{S}$ as its input layer and runs through a number of fully connected layers. Finally, it uses the assortment $\bm{S}$ to create an output gate to ensure $\prob(i|\mathcal{S})=0$ if $i\notin \mathcal{S}$.  We initialize $\bm{z}_{0} = \bm{S}$ as the input layer and for each layer $l=1,...,L,$
\begin{equation}
\bm{z}_{l} = \left(\bm{W}_l\bm{z}_{l-1}+\bm{b}_l\right)^{+}
\label{eq:GAsN}
\end{equation}
where $(\cdot)^+=\max\{\cdot,0 \}$.

\paragraph{Res-Assort-Net (RAsN)}\

The RAsN model incorporates the concept of residual learning, which has demonstrated a remarkable enhancement in image recognition compared to the performance of plain networks \citep{he2016deep}.  As in the GAsN, it is initialized by $\bm{z}_{0} = \bm{S}$ and finally uses the assortment to create an output gate to ensure $\prob(i|\mathcal{S})=0$ if $i\notin \mathcal{S}$. However, for each layer $l=1,...,L$ in RAsN,
\begin{equation}
\bm{z}_{l} = \left(\bm{W}_l\bm{z}_{l-1}+\bm{b}_l\right)^{+}+\bm{z}_{l-1}.
\label{eq:RAsN}
\end{equation}
 Here $\bm{z}_{l} = (z_{l,1},...,z_{l,n}) \in \mathbb{R}^{n}$ is the output of the $l$-th layer and the input for the next layer, $\bm{W}_l\in \mathbb{R}^{n\times n}$ and $\bm{b}_l\in \mathbb{R}^{n}$ are the weight matrix and bias vector, parameters of the neural network model. Note we require the dimensions $n_l=n$ for all $l=1,...,L$ for making the operation of adding $\bm{z}_{l-1}$ in \eqref{eq:RAsN} available.

 Compared to \eqref{eq:GAsN}, the additional (directly added) $\bm{z}_{l-1}\in \mathbb{R}^n$ in \eqref{eq:RAsN}   represents the residual effect.   Intuitively, directly adding it in \eqref{eq:RAsN} can avoid forgetting/distorting the information extracted from previous
layers with extreme $\bm{W}_l$ (e.g., all entries are close to $0$) during the training, which can hopefully further lead to a more robust training process and a better trained network.


\paragraph{Computational Experiment Design}\

We give a preview of our numerical experiment methodology and also demonstrate the model capacity of the two neural choice model architectures  (full details are presented in Appendix~\ref{apnx:no_fea_model}).

\begin{table*}[ht!]
    \centering
        \resizebox{\columnwidth}{!}{%
    \begin{tabular}{c|c|cc|cc|cc|cc}
    \toprule
    &\multirow{2}{*}{\# Samples}&\multicolumn{2}{c|}{MNL}  &\multicolumn{2}{c|}{MCCM}&\multicolumn{2}{c|}{NP}&\multicolumn{2}{c}{MMNL}  \\
        &&$|\mathcal{N}|=20$&$|\mathcal{N}|=50$  &$|\mathcal{N}|=20$&$|\mathcal{N}|=50$ &$|\mathcal{N}|=20$&$|\mathcal{N}|=50$&$|\mathcal{N}|=20$&$|\mathcal{N}|=50$ \\
         \midrule
   Uniform & -- & 2.51 & 3.4 &2.52 & 3.44 & 2.50&3.43&2.48&3.40\\
\midrule
   \multirow{3}{*}{MNL-MLE}  & 1,000&	1.87&2.65&1.96&0.50&1.73&2.50&1.95&2.41\\
   & 5,000&1.86& 2.63& 1.96& 0.50& 1.71& 2.48&1.96&2.37\\
   &100,000&1.86& 2.62&1.95& 0.52& 1.71& 2.47&1.96&2.36\\
 \midrule
  \multirow{3}{*}{MCCM-EM}  & 1,000&1.86&2.67&1.92&0.54&1.67&2.35&1.86&2.27\\
   &5,000&1.86 &2.66 &1.88 &0.53 &1.60 &2.35&1.88&2.23 \\
   & 100,000 &1.87 &2.66	&	1.88&0.51	&	1.62&2.37&1.88&2.22	\\
\midrule
   \multirow{3}{*}{Gated-Assort-Net}  & 1,000&2.00&3.04& 1.97& 0.54& 1.69&2.71&2.01&2.66\\
   & 5,000& 1.89& 2.81& 1.88&0.48&1.55&2.43&1.87&2.30\\
   &100,000& 1.86& 2.63& 1.84& 0.40&1.50&2.23&1.82&2.08\\
\midrule
   \multirow{3}{*}{Res-Assort-Net}  & 1,000&2.05&3.05& 2.03& 0.61&1.80&2.79&2.18&2.96\\
   & 5,000&1.93&2.82&1.91&0.56&1.56& 2.51&1.97&2.53\\
   &100,000& 1.86& 2.64& 1.84&0.41&1.49& 2.22&1.82&2.10\\
\midrule
  Oracle & -- & 1.86&	2.62&	1.82&	0.38&	1.42&	2.13&1.80&2.04\\
   \bottomrule
\end{tabular}}
\caption{Feature-free choice modeling: The training data are generated from the model of MNL, MCCM, NP and MMNL, with randomly generated true parameters. The number of samples refers to the number of training samples, and additional 5,000 samples are reserved to validate the training procedure. The reported numbers are the average out-of-sample cross-entropy (CE) loss over 10,000 test samples and 10 random trials. The row Uniform denotes the performance of a uniform distribution for predicting. MNL-MLE implements the standard maximum likelihood estimation to estimate a MNL model, while MCCM-EM implements the expectation maximization algorithm to fit a MCCM model. GAsN and RAsN both have one hidden fully-connected layer. The row Oracle denotes the CE loss of the true model.}
    \label{tabRP}
\end{table*}

We first generate training data from four models:  MNL, MCCM, NP and MMNL.  We then compare the predictive performance of our two neural choice models, GAsN and RAsN, against three benchmarks and an oracle that has knowledge of the true generative model (equivalently, the negative entropy of the true distribution).   Table~\ref{tabRP} shows the results.

First, both neural choice models benefit from a larger sample size, while the performances of maximum likelihood estimation (MLE) and expectation maximization (EM) do not improve much as the sample size increases. Second, when the true model is MNL, we know the method of MLE is provably asymptotically optimal and this is unsurprisingly verified from our experiment. But when the true model becomes more complex, such as NP and MMNL, the neural network models show their advantage. Third, the learning of a choice model requires a large sample size when the underlying true model becomes more complex. For example, under the true model of MCCM, NP or MMNL, none of these methods give a satisfactory performance when the sample size $m=1000$. Finally, the neural choice models perform consistently well in recovering these true models with a large amount of data, $m=100000$. They have the model capacity to capture complex structures and they are also capable of fitting a simpler true model such as MNL. In comparison, MNL choice model suffers from too little model capacity to capture them. Meanwhile, it is widely acknowledged that the models of NP and MMNL are certainly flexible and can model behaviour realistically, but lack an effective learning/estimation algorithm given the number of variables that are unobservable in their stories, hence making them ill-suited for operational purposes.

\subsection{Theoretical Analysis}
\label{subsec:theorem}
In this subsection, we briefly discuss some theoretical properties of the neural choice model in terms of its parameter estimation. Throughout this subsection, we will focus on the model of GAsN. We remark that the RAsN can be analyzed in the same way. First, we recall that the neural network model can be written as
$$\bm{Y}=g(\bm{S};\bm{\theta}) = (g_1(\bm{S};\bm{\theta}),...,g_n(\bm{S};\bm{\theta}))^\top$$
which maps the assortment $\bm{S} \in \{0,1\}^n$ to the choice probability vector $\bm{Y}\in [0,1]^n$. The parameters $\bm{\theta}$ are estimated through \textit{empirical risk minimization} on a dataset $$\mathcal{D} = \left\{(i_1,\bm{S}_1),....,(i_m,\bm{S}_m)\right\}$$
where $(i_k,\bm{S}_k)$'s are i.i.d. samples from some unknown distribution $\mathcal{P}$. Specifically, we define the \textit{risk} as the negative log-likelihood function (also known as the cross-entropy loss)
$$r((i,\bm{S});\bm{\theta}) \coloneqq -\log g_{i}(\bm{S};\bm{\theta})$$
where $g_{i}(\bm{S};\bm{\theta})$ gives the probability that the $i$-th product is chosen under parameter $\bm{\theta}$.

The estimated parameters $\bm{\theta}$ are obtained by maximum likelihood estimation as follows
$$\hat{\bm{\theta}} \coloneqq \argmin_{\bm{\theta}\in\Theta} \hat{R}_m(\bm{\theta}) \coloneqq \frac{1}{m}\sum_{k=1}^m r((i_k,\bm{S}_k);\bm{\theta})$$
where the function $\hat{R}_m(\cdot)$ is also known as the \textit{empirical risk}.  Here the set $\Theta$ is defined as follows:
$$\Theta \coloneqq \left\{\bm{\theta}:  \|\bm{W}_l\|_{\infty}\leq \bar{W}, \|\bm{b}_l\|_{\infty}\leq \bar{b}, \text{ for } l=1,...,L \right \}$$
where $\bm{W}_l$ and $\bm{b}_l$ are the weight matrix and the bias vector of the $l$-th layer of the neural network.
Here for a matrix $\bm{W}\in \mathbb{R}^{n_1\times n_2}$ we define $\|\bm{W}\|_{\infty}\coloneqq \max_{i=1,.., n_1}\sum_{j=1}^{n_2}|W_{ij}|$ and for a vector $\bm{b}_l\in \mathbb{R}^{n_1}$, $\|\bm{b}_l\|_{\infty}\coloneqq \max_{i=1,.., n_1}|b_i|$. The restriction to this bounded set usually improves both the stability of the training procedure and the generalization performance of the neural network.

Also, we define the \textit{expected risk} as
$$R(\bm{\theta}) = \E\left[r((i,\bm{S});\bm{\theta})\right]$$
where the expectation is taken with respect to $(i,\bm{S})\sim\mathcal{P}.$ The expected risk can be interpreted as the expected negative likelihood on some unseen/new data sample. We further define $\bm{\theta}^*$ as its minimizer:
$$\bm{\theta}^*:=\argmin_{\bm{\theta}\in \Theta}R(\bm{\theta}).$$
The theorem below shows the gap between $R(\hat{\bm{\theta}})$, the expected risk of the learned parameter,  and the optimal risk $R(\bm{\theta}^*)$.
\begin{theorem}
\label{thm:gen_bound}
The following inequality holds with probability $1-\delta$,
$$R(\hat{\bm{\theta}}) \le R(\bm{\theta}^*) + \frac{4n}{\sqrt{m}}\left(\bar{b}\cdot\frac{(2\bar{W})^{L}-1}{2\bar{W}-1}+(2\bar{W})^{L} \sqrt{2\log(2n)}\right)+5C\sqrt{\frac{2\log(8/\delta)}{m}}$$
where $C$ is the upper bound of risk function such that $|r((i,\bm{S});\bm{\theta})|\leq C$ holds for all $(i,\bm{S})\in \mathcal{D}$ and $\bm{\theta} \in \Theta$.
\end{theorem}

Theorem~\ref{thm:gen_bound} shows the excess risk of the empirical risk minimizer $\hat{\bm{\theta}}$. The proof is based on the standard analysis of neural networks \citep{wan2013regularization,neyshabur2015norm,golowich2018size}, but we need to customize it to GAsN due to the special gated operator in the final layer. We defer the proof to Appendix~\ref{apnx:proof}.  By exploring the structure of risk function, we can further bound the distance between the true distribution of choice probability and the estimated choice probability by $\hat{\bm{\theta}}$:

\begin{corollary}
The following inequality holds with probability $1-\delta$,
$$\E\left[\text{KL}(\mathcal{P}_{\bm{S}},g_{\cdot}(\bm{S};\hat{\bm{\theta}}))\right]\leq \E\left[\text{KL}(\mathcal{P}_{\bm{S}},g_{\cdot}(\bm{S};\bm{\theta}^*))\right]+\frac{4n}{\sqrt{m}}\left(\bar{b}\cdot\frac{(2\bar{W})^{L}-1}{2\bar{W}-1}+(2\bar{W})^{L} \sqrt{2\log(2n)}\right)+5C\sqrt{\frac{2\log(8/\delta)}{m}}$$
where $C$ is defined as in Theorem~\ref{thm:gen_bound}, $\mathcal{P}_{\bm{S}}$ is defined as the distribution of choice probability given assortment $\bm{S}$, $\text{KL}(\cdot,\cdot)$ is defined as the Kullback–Leibler divergence over two choice probability distributions and expectation is with respect to $\bm{S}$.
\end{corollary}
\begin{proof}
Given an assortment $\mathcal{S}$, the Kullback–Leibler divergence over the true choice probability $\mathcal{P}_{\bm{S}}$ and $g_{\cdot}(\bm{S};\bm{\theta})$ is
$$\text{KL}(\mathcal{P}_{\bm{S}},g_{\cdot}(\bm{S};\bm{\theta})) =\sum_{i\in \mathcal{S}} P(i|\mathcal{S})\left(\log(P(i|\mathcal{S}))-\log(g_i(\bm{S};\bm{\theta}))\right),$$
where $P(i|\mathcal{S})$ is the choice probability of $i$-th product given assortment $\mathcal{S}$.

Note that
$$\E\left[-\sum_{i\in \mathcal{S}} P(i|\mathcal{S})\log(g_i(\bm{S};\bm{\theta}))\right]=R(\theta),$$
where expectation is with respect to the $\mathcal{S}$, with Theorem~\ref{thm:gen_bound}, we complete our proof.
\end{proof}

The term $\E\left[\text{KL}(\mathcal{P}_{\bm{S}},g_{\cdot}(\bm{S};\bm{\theta}^*))\right]$ is usually referred as \textit{approximation error}, which depends on the capacity of function family $\{g_{\cdot}(\bm{S};\bm{\theta}): \bm{\theta} \in \Theta\}$, and the remaining two terms with order $O(1/\sqrt{m})$ are usually together referred as \textit{estimation error}, which depends on the number of training samples $m$ and also the capacity of function family through $W$, $\Bar{b}$ and $L$. In general, with fixed training samples, a larger function family $\{g_{\cdot}(\bm{S};\bm{\theta}): \bm{\theta} \in \Theta\}$ (i.e., deeper and wider networks) will reduce the approximation error but will increase the estimation error.

\subsection{Feature-based Neural Choice Models}
\begin{figure}[ht!]
    \centering
    \includegraphics[scale=0.25]{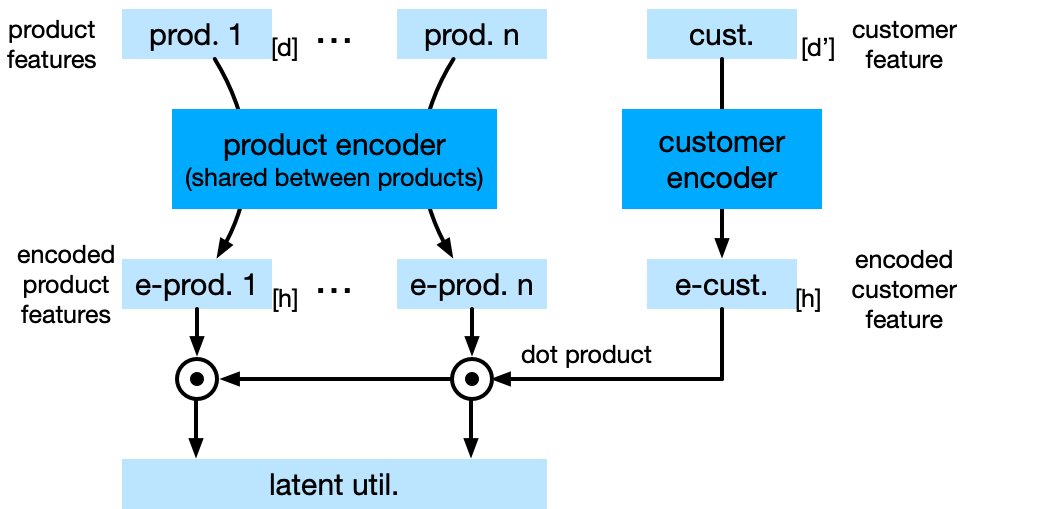}
    \caption{Feature encoder: Product features and customer features are encoded and then taken inner product to obtain the latent utility for each product. $d$ is the dimension of the product features, and $d'$ is the dimension of the customer features. Both features are encoded to $h$-dimensional latent features, and the latent utilities are obtained by the inner product of the corresponding latent features.}
    \label{fig:feat_encode}
\end{figure}

Now we extend our neural network models to the setting  with features. We first make a distinction between product feature and customer feature:

\begin{itemize}
\item Product feature: there is a feature vector $\bm{f}^P_i\in\mathbb{R}^d$ associated with each product $i\in\mathcal{N}$. We also refer to these features as \text{static} features as they usually remain the unchanged over all the samples in the training data $\mathcal{D}$.
\item Customer feature: Sample in $\mathcal{D}$ represents different customers, each of which is associated a feature vector $\bm{f}^C_k\in\mathbb{R}^{d'}$ for $k=1,...,m$. With these features, the dataset is augmented as
$$\mathcal{D} = \{(i_k, \mathcal{S}_k, \bm{f}^C_k), k=1,...,m\}.$$
We refer to these features as \textit{dynamic} features as they may vary across different samples. Modeling-wise, if there are product features that change over time, we can simply view them as dynamic features.
\end{itemize}

\begin{figure}[ht!]
    \centering
    \begin{subfigure}[b]{0.25\textwidth}
        \centering
        \includegraphics[height=3.5in]{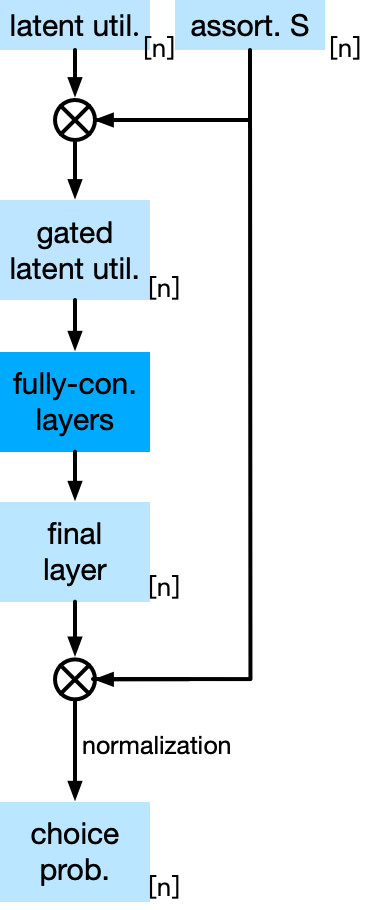}
        \caption{Gated-Assort-Net(f)}
    \end{subfigure}%
    \begin{subfigure}[b]{0.25\textwidth}
        \centering
        \includegraphics[height=3.5in]{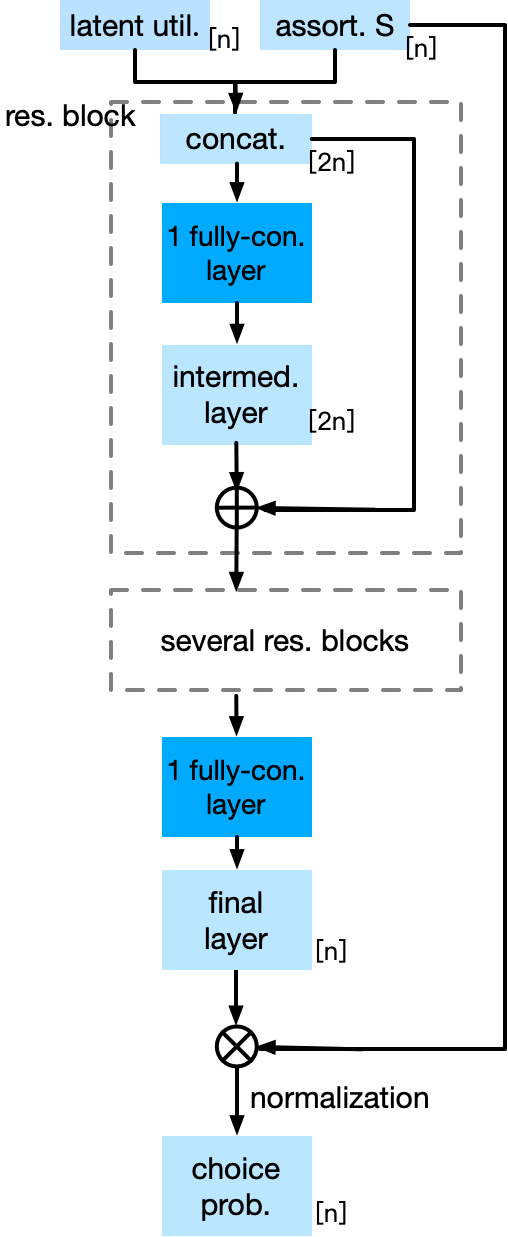}
        \caption{Res-Assort-Net(f)}
    \end{subfigure}
    \caption{Choice networks with features: The networks take (i) the latent utilities from feature encoder (in Figure~\ref{fig:feat_encode}) and (ii) the assortment vector as inputs.}
    \label{fig:net_feat}
\end{figure}

Figure~\ref{fig:feat_encode} and Figure~\ref{fig:net_feat} describe our choice networks with features. The neural networks inherit the architectures in the feature-free setting, but both take an additional vector of inputs which we call  latent utilities. Both feature-based networks of GAsN(f) and RAsN(f) encode product features and customer features to obtain one latent utility for each product. The product encoder is an $L$-layer fully-connected network shared by all the products ($L$ is $1$ or $2$ in our experiments). By default, the number of nodes in the intermediate layers of the encoder is the same as the dimension of the encoder's input. When there is no available customer feature, we can simply treat the customer feature as 1 for all the samples.

\begin{table}[ht!]
\centering
\begin{tabular}{c|cc}
\midrule
  &  MNL(f)  & MCCM(f) \\
\midrule
 MNL(f)-MLE & 2.903 & 2.212\\
 \midrule
Gated-Assort-Net &  2.921& 2.010\\
Gated-Assort-Net(f) & 2.918 & 2.010\\
\midrule
Res-Assort-Net & 2.914 & 1.988\\
Res-Assort-Net(f) & 2.914 & 1.984\\
\midrule
Oracle & 2.900 & 1.932\\
\bottomrule
\end{tabular}
\caption{Choice modeling with only product (static) features: The training data (with $n=50$, $m=100,000$, and $d=5$) are generated from the feature-based version of MNL and MCCM. All networks have two-layer structures (for the networks with features, the encoder part has one layer). The reported numbers are the average out-of-sample cross-entropy (CE) loss over 10 random trials. The benchmark method implements MLE for the featured-based MNL model. More details about experiment setup can be found in Appendix B.1. }
\label{tabFeat1}
\end{table}

\paragraph{Unnecessity of Static Features}\

Table~\ref{tabFeat1} presents a synthetic experiment with only product (static) features to illustrate a new insight for choice modeling with features.  The experiment first ignores all the product features and trains GAsN and RAsN as in the feature-free setting. Then it takes into account the product features and trains the feature-based version of the two neural networks. We find that both versions of the networks are capable of recovering the true model from the comparison against the oracle performance, and that the two versions of each neural network achieve a similar performance. This might seem counter-intuitive at first sight, because the product features may contain useful information about the product's utility, and ignoring the features may accordingly hurt the performance. However, we argue that the product utilities are essentially encoded in the choice data $(i_k, \mathcal{S}_k)$'s but not in the product features. In fact, the experiment shows that such utilities can be learned by the neural networks of GAsN and RAsN without explicitly accessing the product features. An alternative interpretation is that the product utility here should be viewed as a population-level utility that reflects the preference or popularity of a product over the whole population. Things become different when customer features are also available. In that case, the choice behaviour of each data sample is determined by the personalized utilities, and thus the product and customer features become indispensable.

We emphasize that whether to include the (dynamic) customer features in choice modeling is determined not by the availability of the features, but by the downstream application. For example, the assortment optimization problem aims to find an assortment that maximizes profits under a choice model. For a brick-and-mortar store or an online retailer where the personalized assortment is not allowed, customer features should not be used to fit the choice model even if they are available, because one is more interested in the population-level choice behavior.   It also underscores the potential of feature-free networks of GAsN and RAsN despite their simple architectures.

\subsection{Comparison on Real Datasets}
In this subsection, we test the predictive performances of our neural choice model with several benchmarks on four real datasets, two private datasets from revenue management industries without features and two public ones with features.    We repeat all experiments $10$ times, choosing the train/validation/test sets independently and randomly, and report the averaged results below.  In Section~\ref{sec:OPT_test} we perform extensive numerical experiments on synthetic datasets to jointly test prediction and optimization performance.

\subsubsection{Feature-free Neural Choice Model}

 We conduct numerical experiments on two private real datasets: airline data and retailing data. The airline data contains three flight markets and the offered products are different bundles, i.e., bundle of seat, food etc. The offered assortments are decided by both the selling strategy and remaining seats. The retailing data is collected from a  fashion-retailing company, where the products are clothes (from a same category and thus most customers  purchase at most one item) and the assortments are decided by the inventories. Both datasets  contain  records of purchase transactions only, so we add no-purchases as in \cite{csimcsek2018expectation}: for each purchase, we create four no-purchases with the same assortment as the original transaction but as choosing the outside no-purchase option.  The summary statistics of the dataset is given in Table~\ref{tab:No_fea_data}.

\begin{table}[ht!]

    \centering
    \begin{tabular}{c|c|c|c|c}
    \toprule
       &Flight $1$&Flight $2$&Flight $3$& Retailing \\
\midrule
\# Samples&90634&37075&70600&17200\\
\# Products&25&25&25&20\\
\# Train Samples&60000&30000&60000&10000\\
\# Validate Samples&1000&1000&1000&1000\\
\# Test Samples&10000&5000&10000&5000\\
 \bottomrule
    \end{tabular}
    \caption{Description of flights and retailing data.}
    \label{tab:No_fea_data}
\end{table}

 The testing result is summarized in Table~\ref{tab:Air_result}, where the reported numbers are out-of-sample cross-entropy loss of two benchmark methods (with the same configurations as in Table~\ref{tabRP}) and the two neural network models with one layer.

 We observe that our neural networks achieve the best performance on these datasets where the generative model is unknown and likely complex. One interesting observation is unlike other numerical results in this paper, MNL-MLE is better than MCCM-EM for the airline markets. One potential reason for the relatively poor performance of MCCM-EM might be the complicated assortment generation: in Subsection~\ref{subsec:OOD}, we numerically show the performances of MCCM-EM is correlated with the assortment distributions, which are special in the airline markets due to their nested booking-limits selling strategy.

\begin{table}[ht!]
    \centering
    \begin{tabular}{c|c|c|c|c}
    \toprule

       &Flight $1$&Flight $2$&Flight $3$&Retailing   \\
\midrule
MNL-MLE
&2.798&2.490&2.547&1.073\\
MCCM-EM
&3.972&3.470&3.596&1.063\\

\midrule
Gated-Assort-Net & 2.738 & 2.436 & \textbf{2.489} &1.063\\
Res-Assort-Net & \textbf{2.734} & \textbf{2.434} & 2.490 & \textbf{1.060}\\
 \bottomrule
    \end{tabular}
    \caption{Performance on the airline data and retailing data. The reported numbers are out-of-sample CE loss.}
    \label{tab:Air_result}
\end{table}

\subsubsection{Feature-based Neural Choice Model}
For the case of estimation with features, we perform numerical experiments on two public datasets -- SwissMetro  and Expedia Search. These two datasets contain dynamic customer/product features; that is, the features associated with each training sample may be different.

The SwissMetro dataset is a public survey dataset to analyze traveller preference among three transportation modes, and the Expedia dataset consists of hotel search records and the associated customer choices. More details of datasets, experiment setup and benchmarks can be found in Appendix~\ref{sec:apx_fea_model}.


\begin{table}[ht!]
    \centering
    \begin{tabular}{c|c|c}
    \toprule
    &SwissMetro  &Expedia\\
    \midrule
    MNL-MLE&0.883&2.827\\
MNL(f)-MLE&0.810&2.407\\
\midrule
TasteNet
&0.681&2.436\\
DeepMNL
&0.741&2.374\\
Random Forest
&0.633&2.458\\
\midrule
Gated-Assort-Net(f)
&\textbf{0.598}&2.403\\
Res-Assort-Net(f)&
0.607&\textbf{2.325}\\
   \bottomrule
    \end{tabular}
    \caption{Performance on the SwissMetro and Expedia. The reported numbers are out-of-sample CE loss.}
    \label{tabExp}
\end{table}

We implement a number of benchmark models: MNL, feature-based MNL, TasteNet \citep{han2020neural}, DeepMNL \citep{wang2020deep, sifringer2020enhancing}, random forest \citep{chen2021estimating}. The MNL and feature-based MNL are learned by MLE.  TasteNet and DeepMNL model the product utility $u_i$ as a function of the product and customer features and map the utility to choice probability via MNL \eqref{eqn:MNL}. TasteNet assumes the utility $u_i$ is a linear function of product features and uses the customer feature to determine the coefficients of the linear function, while DeepMNL concatenates the customer and product features and feeds both into a fully-connected network to obtain the utility $u_i$. Both models bring assortment effects only through the MNL part.

From the results in Table~\ref{tabExp}, our two neural networks give a better performance than the benchmark models. Even though the benchmark models include neural networks in this case, the key difference is that our models explicitly take the assortment as an input. This experiment result shows that even with the presence of product and customer features, incorporating assortment effects can help to better explain customer choices. Comparing the two datasets, the Gated-Assort-Net(f) performs better on SwissMetro, while Res-Assort-Net(f) performs better on Expedia. Recall that the Res-Assort-Net(f) makes a stronger usage of the assortment vector throughout the architecture than Gated-Assort-Net(f). Accordingly, one explanation can be that for the transportation setting, customer choices are less affected by the available options but more by their personal preference, but for the hotel search, the customer choices are more affected by the provided assortment, which gives Res-Assort-Net(f) more of an advantage. As for the random forest model, it is trained as a discriminative model using both features and assortment as input, so we believe it also has better potential with some further model recalibration.

\section{Assortment Optimization for Neural Choice Model}
\label{sec:OPT}
An important downstream task after the estimation of a choice model is the assortment optimization problem, where the seller aims to find an assortment that maximizes the profits
\begin{equation}
\max_{\mathcal{S}\subset \mathcal{N}} \text{Rev}(\mathcal{S}) \coloneqq \sum_{i\in\mathcal{N}}\mu_{i} \prob(i|\mathcal{S})
\label{assort_opt}
\end{equation}
where $\mu_i$ is the profits/revenue of selling one unit of product $i$. The problem usually has additional space or cardinality constraints.


The literature on assortment optimization fixes an underlying choice model $\mathcal{M}$ and devises algorithms with provable performance guarantees. However,  in practice,  the underlying choice model is unknown a priori and a model has to be selected first and its parameters estimated from data $\mathcal{D}$. In this case, the suboptimality of a proposed assortment may come from two sources: (i) the approximation error: the inaccuracy of the choice model estimation and (ii) the optimization error: the optimality gap induced by solving the assortment optimization problem under the estimated model. The following bound shows the intuition:
\begin{align}
\mathrm{Rev}(\mathcal{S}^*) - \mathrm{Rev}(\hat{\mathcal{S}}) \le \underbrace{\bar{\mu}\cdot \mathrm{dist}(\mathcal{M}, \hat{\mathcal{M}})}_{\mathrm{Model\ Approx. Error}} + \underbrace{\mathrm{Opt.\ Gap}}_{\mathrm{Opt.\ Error}}.\label{ineq_decompose}
\end{align}
Here $\mathcal{M}$ denotes the true choice model and $\hat{\mathcal{M}}$ is the choice model estimated from the data $\mathcal{D}$.  $\mathcal{S}^*$ is the optimal assortment obtained from \eqref{assort_opt} using the true model $\mathcal{M}$, and $\hat{\mathcal{S}}$ is obtained from $\hat{\mathcal{M}}$ via some assortment optimization algorithm, where the term $\mathrm{Opt.\ Gap}$ refers to the suboptimality induced by the algorithm for assortment optimization. The function $\mathrm{dist}(\cdot, \cdot)$ refers to some (pseudo-)distance function between two models, and $\bar{\mu}\coloneqq \max_{i\in\mathcal{N}} \mu_i$.

We use the inequality \eqref{ineq_decompose} to emphasize the decomposition of the revenue gap into the approximation error and the optimization error. The existing works on assortment optimization have striven to improve the optimization error, while assuming the approximation error is zero. We make the case that both errors should be taken into account when measuring the performance of an assortment algorithm as the firm is interested in the combined performance.  In  light of this, we measure the performance of our neural choice model illustrating the tradeoff between these two errors.

\subsection{Integer Programming Formulation of the Neural Choice-based Assortment Optimization}
In this section we present a mixed integer programming (MIP) formulation for the assortment optimization problem under the neural choice model. Denote $\bm{\mu} = (\mu_1,...,\mu_n)^\top\in \mathbb{R}^n$ as the revenue vector as in \eqref{assort_opt}. Throughout this section, we will focus on the model of GAsN. We remark that the RAsN gives a similar MIP formulation and numerical performance. We formulate the assortment optimization problem under the GAsN as follows,

\begin{align}
 \max \ \ &  \frac{\sum_{i=1}^n\mu_iz_{0,i}\exp(z_{L,i})}{\sum_{i=1}^n z_{0,i}\exp(z_{L,i})} \label{MIP_OPT}\\
\text{s.t. }\
& \bm{z}_{l}-\tilde{\bm{z}}_l = \bm{W}_l \bm{z}_{l-1}+\bm{b}_l, \text{  for } l=1,...,L,\nonumber\\
&\bm{0}\le \bm{z}_{l}\leq M\bm{\zeta}_{l}, \ \bm{0}\le \tilde{\bm{z}}_{l}\leq M(\bm{1}-\bm{\zeta}_{l}), \text{  for } l=1,...,L,\nonumber\\
& \bm{\zeta}_{l}\in \{0,1\}^{n_{l}}, \text{  for } l=1,...,L,\nonumber\\ &\bm{z}_0=(z_{0,1},...,z_{0,n})^\top \in \{0,1\}^n,\nonumber
\end{align}
where the decision variables are $\bm{z}_l\in\mathbb{R}^{n_{l}}$, $\tilde{\bm{z}}_l\in\mathbb{R}^{n_{l}}$, $\bm{\zeta}_l\in \{0,1\}^{n_{l}}$, all for $l=1,...,L$, and $\bm{z}_0\in\{0,1\}^{n}.$ The decision variables $\bm{z}_0$ for the input layer provide a binary representation of the assortment decision, and $\bm{z}_l$ represents the values of the intermediate layers of the neural network. The equality constraints describe the forward propagation of the neural network layers, and the inputs $\bm{W}_l$ and $\bm{b}_l$ are the parameters of the neural network learned from the data. The objective function describes the mapping from the final layer of the network to the revenue under the assortment $\bm{z}_0.$

The (auxiliary) decision variables $\tilde{\bm{z}}_l$ and $\bm{\zeta}_{l}$ implement the big-M method \citep{fischetti2017deep, conforti2014integer} where the value of $M$ is a large positive number obtained by some prior estimate. These two jointly ensure that
$$\bm{z}_l=\left(\bm{W}_l \bm{z}_{l-1}+\bm{b}_l\right)^+.$$
It is well known that the choice of $M$   affects the running time of solving the MIP \citep{belotti2016handling}. In our numerical experiments, we follow the methods in \citep{fischetti2017deep} to tighten the upper bound. Also, additional linear constraints can be added to incorporate the cardinality or capacity constraints. When there is a no-purchase option, one may introduce an additional constraint of $z_{0,n}=1$ ($n$ being the index of the no-purchase option).

We remark that the constraints of the MIP \eqref{MIP_OPT} are all linear, and the objective function resembles that of the MNL model. For a better numerical performance, we replace the exponential function in the objective function with its second-order approximation $\exp(x)\sim 1+x+x^2/2$ \citep{banerjee2020exploring}. Though the problem \eqref{MIP_OPT} cannot be solved in polynomial time as is the case for some of the other models such as MNL and MCCM (when there is no capacity constraint), it does provide a formulation that allows us to take advantage of off-the-shelf integer programming solvers. 

\subsection{Numerical Study of Joint Estimation and Optimization}\label{sec:OPT_test}

In this subsection, we provide a comparative study of the four models (MNL, MCCM, 1 and 2-layer GAsN)  along with their corresponding estimation methods and assortment optimization algorithms by running them on a total of 9600 different synthetic problem instances generated by a mix of ground-truth strategies. Since there are no computationally efficient estimation methods for NP and MMNL choice models,  we do not include them here.   Each problem instance is a combination of historical sales and revenue data, and additional capacity constraints (if any), and the algorithm(s) output a recommended assortment following this data-to-decision pipeline:
$$\mathcal{D}=\{(i_k, \mathcal{S}_k), k=1,...,m\}, (\mu_{1},...,\mu_{n}), \text{Constraints} \ \rightarrow \text{ Recommended Assortment }\hat{\mathcal{S}}.$$
The numerical experiment is to test performance where the true model is unknown and has to be estimated from data.

We generate the training data $\mathcal{D}$ from four different choice models as ground truth, MNL, MCCM, NP, and MMNL with the number of training samples fixed at $m=30,000$ for all the trials. The exact optimal assortment $\mathcal{S}^*$ is computed with the knowledge of the generative model, and then the recommended assortment $\hat{\mathcal{S}}$ is evaluated by the optimality ratio
$$\text{Opt. Ratio} \coloneqq \frac{\text{Rev}(\hat{\mathcal{S}})}{\text{Rev}(\mathcal{S}^*)}.$$
For each problem instance, we consider both unconstrained and constrained settings and implement several benchmark methods as follows. For the estimation part, we implement the maximum-likelihood estimation to estimate the parameters of the MNL model, and the expectation-maximization algorithm to estimate those of MCCM.  We provide more details of the experiments in Appendix~\ref{apnx:assort_opt}. Below are the assortment optimization algorithms:
\begin{itemize}
\item Unconstrained: The estimated choice model is  used to solve for a recommended assortment through the assortment optimization specific for the method.
\item Constrained:
\begin{itemize}
\item  Revenue-ordered policy (RO): All the $n$ products are ordered in an decreasing order by their revenues to form a nested set of assortments and the feasible assortment with largest expected revenue is picked from that nested set \citep{talluri2004revenue}.
\item MNL-MIP: A mixed integer programming is solved based on the estimated MNL model \citep{mendez2014branch}.
\item ADXOpt: The ADXOpt algorithm proposed by \cite{jagabathula2014assortment} is a greedy algorithm to solve assortment optimization under a general choice model. We implement the algorithm on the estimated Markov chain choice model. In our implementation, we increase the removal limit (the maximal number of times for a product to be removed from the assortment) from $1$ in \citep{jagabathula2014assortment} to $5$ so as to boost the algorithm performance.
\item  Markov Chain Capacity-Assort (MC-CA): We implement the algorithm given in \cite{desir2020constrained} on the estimated Markov chain choice model.
\item Neural Network Optimization (NN): We implement the MIP \eqref{MIP_OPT} on the fitted GAsN with $1$ or $2$ layers (denoted as NN(1) and NN(2)).
\end{itemize}
\end{itemize}

The results of our numerical experimets are given in Table~\ref{tab:assortOPT_uni_NC} and Table~\ref{tab:assortOPT_uni_C}.
In general, the neural network with MIP formulation performs well in all underlying choice models achieving  the best performance in several cases (averaged over different numbers of products), with and without a capacity constraint; and under MNL choice model, the (averaged) approximation ratios of NN(1) are above $95\%$, both with and without constraint.

The advantage of a neural network is more apparent when the generative models are complicated: with data generated by NP and MMNL and with constraints, the averaged approximation ratios of NN are $92.55\%$ and $84.79\%$, where the best averaged performance of the other methods are $88.28\%$ and $67.14\%$. The relatively poor performances of the other methods can be due to the limited expressive power or the hardness of the subsequent optimization. While the structural limitations of MNL are well known, and hence is not surprising that it cannot approximate NP and MMNL generated data. For the MCCM, \cite{blanchet2016markov} show that the worst case error bound of MCCM to approximate a choice probability of an assortment in MMNL is negatively correlated to ``the maximum probability that the most preferable product for a random customer from any segment belongs to that assortment'' (Theorem 4.2). Thus, when the offered assortment includes many ``preferable products'', the estimated MCCM may not approximate the underlying MMNL choice model well.   Also, constraints   make approximation ratios worse by making the optimization harder; this can be observed by comparing Table~\ref{tab:assortOPT_uni_NC} and Table~\ref{tab:assortOPT_uni_C}.

Comparing NN(1) with NN(2), NN(1)'s average approximation ratios are better than NN(2)'s in all cases except MCCM with constraint. In fact, we observe the overfitting of NN(2) during the training in some cases, which is a plausible explanation for its worse relative performance.

In summary, when trying to model and optimize the assortment optimization problems, our finding is that, when evaluating joint estimation and assortment optimization, a shallow network may be preferable to a deep neural network, as it reduces overfitting and leads to a more manageable optimization problem. Of course, if the its estimation error is large, one can try adding more layers to enlarge the capacity and expressive power of the neural network.

\begin{table*}[ht!]
\centering
\begin{tabular}{c|ccc|c|ccc|c}
\toprule
Model&\multicolumn{4}{c|}{MNL}&\multicolumn{4}{c}{MCCM} \\
$|\mathcal{N}|$&20&40&60&Avg.&20&40&60&Avg.\\
\midrule
MNL (MLE)
&99.94&99.36&99.11&99.47&93.54&94.26&94.59&94.13\\
MCCM (EM)
&99.70&99.44&99.65&\textbf{99.60}&93.88&96.80&97.77&96.15 \\
NN (1-layer)
&99.71&98.49&97.80&98.67&96.93&95.78&97.68&\textbf{96.80}\\
NN (2-layer)
&98.03&91.14&86.60&91.92&96.60&90.46&91.83&92.96\\
\bottomrule
\end{tabular}
\vspace{0.4cm}

\begin{tabular}{c|ccc|c|ccc|c}
\toprule
   Model&\multicolumn{4}{c|}{NP}&\multicolumn{4}{c}{MMNL} \\
$|\mathcal{N}|$&20&40&60&Avg.&20&40&60&Avg.\\
\midrule
MNL (MLE)
&91.10&92.17&92.88&92.05&79.87&88.59&92.17&86.88\\
MCCM (EM)
&92.89&92.88&92.99&92.92&84.36&89.24&91.12&88.24\\
NN (1-layer)
&95.36&93.65&93.84&\textbf{94.28}&89.22&90.81&91.35&\textbf{90.46}\\
NN (2-layer)
&95.80&87.00&78.65&87.15&92.15&83.40&75.28&83.61\\
\bottomrule
\end{tabular}
\caption{The approximation ratios ($\times 100\%$) of heuristics for assortment optimization without capacity constraint. Each reported value is averaged over $100$ trials.}
\label{tab:assortOPT_uni_NC}
\end{table*}

\begin{table}[ht!]
    \centering
    \begin{tabular}{cc|ccc|c|ccc|c}
    \toprule
    \multicolumn{2}{c|}{Model}&\multicolumn{4}{c|}{MNL} &\multicolumn{4}{c}{MCCM} \\
\multicolumn{2}{c|}{$|\mathcal{N}|$}&20&40&60&Avg.&20&40&60&Avg.\\
\midrule
\multirow{2}{*}{MNL (MLE)}&RO
&83.27&76.82&77.93&79.34&77.44&74.70&90.36&80.83\\
&MIP
&99.92&99.36&99.02&\textbf{99.43}&89.48&92.01&93.22&91.57\\
\midrule
\multirow{2}{*}{MCCM (EM)}&MC-CA
&92.25&89.98&88.01&90.08&85.83&87.33&95.05&89.40\\
&ADXOpt
&96.91&96.89&96.04&96.61&90.26&92.07&94.48&92.27 \\
\midrule
NN (1-layer)&MIP
&99.14&98.73&98.85&98.91&93.58&92.03&92.71&92.77\\
NN (2-layer)&MIP
&97.49&96.20&94.33&96.01&93.17&93.51&97.08&\textbf{94.59}\\
\bottomrule
\end{tabular}
\vspace{0.4cm}

\begin{tabular}{cc|ccc|c|ccc|c}
    \toprule
    \multicolumn{2}{c|}{Model}&\multicolumn{4}{c|}{NP}&\multicolumn{4}{c}{MMNL} \\
\multicolumn{2}{c|}{$|\mathcal{N}|$}&20&40&60&Avg.&20&40&60&Avg.\\
\midrule
\multirow{2}{*}{MNL (MLE)}&RO
&77.82&78.03&74.47&76.77&59.81&47.37&45.46&50.88\\
&MIP
&86.86&90.37&87.20&88.14&75.39&65.44&60.60&67.14\\
\midrule
\multirow{2}{*}{MCCM (EM)}&MC-CA
&84.96&85.06&84.49&84.84&70.54&55.26&56.17&60.66\\
&ADXOpt
&89.02&88.81&87.02&88.28&73.60&61.70&59.26&64.85\\
\midrule
NN (1-layer)&MIP
&93.90&93.28&90.47&\textbf{92.55}&92.76&84.99&76.63&\textbf{84.79}\\
NN (2-layer)&MIP
&94.80&91.27&87.53&91.20&91.58&77.59&68.28&79.15\\
\bottomrule
\end{tabular}
\caption{The approximation ratios ($\times 100\%$) of heuristics for assortment optimization with capacity constraint. Each reported value is averaged over $100$ trials.}
\label{tab:assortOPT_uni_C}
\end{table}

\section{Extensions and Discussions}
\label{sec:extension}
In this section, we discuss several properties of neural choice models gained from our experience with the experiments, including robustness under the assortment distribution's shift in the testing data and the effects on the depth and width of the networks. We also list some tricks for deploying neural choice models when we need to add new products or the training samples are not sufficient. We defer the details of all experiments in this section to Appendix~\ref{apnx:extension}.

\subsection{Assortment Distribution Shift in the Training Data}
\label{subsec:OOD}
A very practical but often neglected aspect of learning a choice model is the assortment distribution in the training data $\mathcal{D},$ i.e., the distribution of $\mathcal{S}_k$'s. Figure~\ref{fig:EMMC-assort} gives empirical evidence that the learning of an MCCM using the EM algorithm (implementing \citep{csimcsek2018expectation}) is affected by the assortment size of the training samples. Intuitively, a smaller assortment size means it will take longer for the Markov chain to hit the exiting state, and thus the E-step needs to impute larger unobserved quantities which may result in larger variance. Alternatively, we find the performance of the neural network model is quite robust in terms of the generalization on the assortment domain. Specifically, in Appendix~\ref{apnx:OOD}, we train the neural network model with one distribution to generate the assortments $\mathcal{S}_k$'s (in training data $\mathcal{D}$) and find that it performs surprisingly well on the test data where the assortments are generated from a different assortment generating distribution (an aspect of out-of-domain performance).

\begin{figure}[ht!]
\centering
\includegraphics[scale=0.4]{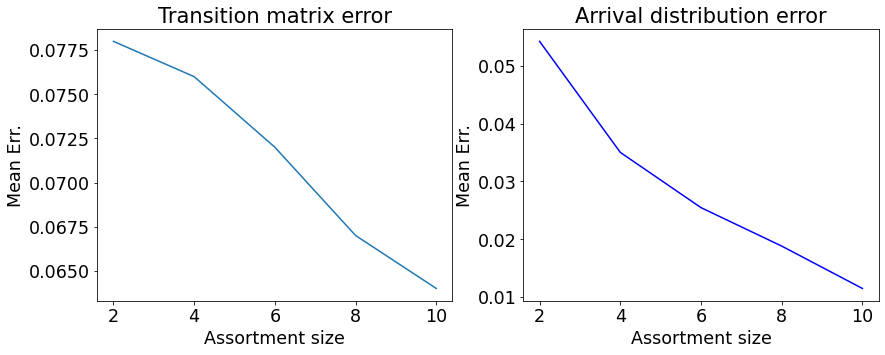}
\caption{Smaller assortment size leads to larger estimation error. In the experiment, the total number of products $n=20,$ and the training data consists of $m=10,000$ samples with a fixed-size assortment $\|S_{k}\|$. The x-axis gives the assortment size and the y-axis represents the mean error. The plotted curves are based on an average over 10 independent trials. }
\label{fig:EMMC-assort}
\end{figure}

\subsection{Depth and Width of the Neural Networks}
\label{subsec:depth}
Throughout the paper, we use one-layer or two-layer neural networks for our neural choice models, i.e., $L=1$ or $2$. We find this provides a sufficiently good fit for both synthetic and real data. For the number of neurons/nodes in the intermediate layers, we set $n_l=n$ for all $l$'s by default. In Appendix~\ref{apnx:DepWidNN}, we provide more numerical results that the performance of the neural choice models remain stable with more number of layers and wider width for the intermediate layers. Generally, it is not recommended to have many layers for the neural network, as it results in a bad landscape for the loss function and consequently bad local minima that do not generalize well \citep{keskar2016large}.
\subsection{Network Augmentation with Warm Start}
\label{subsec:warm_start}
One drawback of the models proposed in this paper is that they all require the universe of available products to be fixed, i.e., $\mathcal{N}=\{1,...,n\}$ is fixed. When an additional set of products $\mathcal{N}'=\{n+1,...,n'\}$ is available, the model has to be retrained on the new product set of $\mathcal{N} \cup \mathcal{N}'$. One option is to initialize the weights of the new network according to the previously well-trained network of $\mathcal{N}$ as a warm start. Figure~\ref{fig:WarmStart} shows that the well-trained network of $\mathcal{N}$ can provide a good warm start for retraining the new network.

\begin{figure}[ht!]
    \centering
    \begin{subfigure}[b]{0.5 \textwidth}
        \centering
        \includegraphics[width=1.\textwidth]{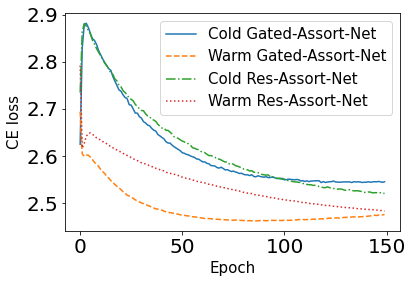}
    \end{subfigure}%
    \begin{subfigure}[b]{0.5 \textwidth}
        \centering
        \includegraphics[width=1.\textwidth]{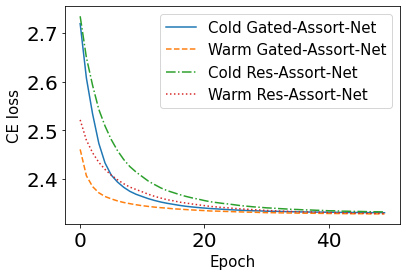}
    \end{subfigure}
    \caption{Validation losses using cold start and warm start: The left panel uses  $2,000$ samples for retraining and the right panel uses $100,000$ samples.}
    \label{fig:WarmStart}
\end{figure}

\subsection{Neural Choice Model as a Meta Choice Model}
\label{subsec:meta}
From our numerical results in Table~\ref{tabRP}, we gain a new perspective, namely, to use neural choice model as a meta choice model. Table~\ref{tabRP} shows how the neural choice model can benefit from large data size but can also be outperformed when the data size is small. In this case, we can use the  estimated choice model with the smallest validation error to generate synthetic data to re-train another neural choice model and fine-tune it by the real training data. The pseudo code is given in Algorithm~\ref{alg:meta}. Intuitively, the strong representation ability of the neural choice model can make its performance arbitrarily close to the well-fitted generative model in the re-training step and can be further improved by real data in the fine-tuning step.   Since the re-training step is indeed learning the trained choice model $\hat{\mathcal{M}}_{k^*}$ with the \textit{best} validation error from all other trained models $\hat{\mathcal{M}}_{k}$,  which has the  \textit{best} training error of the corresponding choice model class $\mathcal{M}_{k}$, we name this algorithm \textit{Learning the Best of the Bests}. In this way, the neural choice model is not trained as a competitor to the existing choice models, but more as a complementary meta choice model. In Appendix~\ref{apnx:Hotel}, we implement Algorithm~\ref{alg:meta} on a public Hotel data \citep{bodea2009data} as an example.

\begin{algorithm}[ht!]
\caption{Learning the Best of the Bests}
\label{alg:meta}
\begin{algorithmic}[1]
\State Input: Training data $\mathcal{D}_{\text{train}}$, validation data $\mathcal{D}_{\text{val}}$, a set of choice models $\{(\mathcal{M}_k, \mathcal{L}_k)\}_{k=1}^K$ where each choice model $\mathcal{M}_k$ is equipped with a learning algorithm $\mathcal{L}_k$. $(\mathcal{M}_1, \mathcal{L}_1)$ corresponds to neural choice model.
\State \textcolor{blue}{\%\% Training phase}
\For {$k=1,...,K$}
\State \%\% We omit the training details such as cross-validation
\State Fit the $k$-th candidate choice model using dataset $\mathcal{D}_{\text{train}}$ and algorithm $\mathcal{L}_k$
\State Denote the trained model as $\hat{\mathcal{M}}_k$
\EndFor
\State \textcolor{blue}{\%\% Comparison phase}
\State Compute the test performance of all $\hat{\mathcal{M}}_k$'s using $\mathcal{D}_{\text{val}}$
\State Let $k^*$ be the index of the model with best test performance
\If {$k^*=1$}
\State Let $\mathcal{M}_{\text{final}}=\hat{\mathcal{M}}_1$
\Else
\State \textcolor{blue}{\%\% Re-training and fine-tuning}
\State Generate $m'(=100,000)$ new training samples as $\mathcal{D}_{\text{train}}'$ from the model $\hat{\mathcal{M}}_{k^*}$
\State Use $\mathcal{D}_{\text{train}}'$ to train (re-train) a neural choice model $\hat{\mathcal{M}}_\text{tmp}$ from scratch
\State Use $\mathcal{D}_{\text{train}}$ to further train (fine-tune) $\hat{\mathcal{M}}_\text{tmp}$ and obtain $\hat{\mathcal{M}}_1'$
\State Let $\mathcal{M}_{\text{final}}$ be the better one among $\hat{\mathcal{M}}_1'$ and $\hat{\mathcal{M}}_{k^*}$
\EndIf
\State Output: $\mathcal{M}_{\text{final}}$
\end{algorithmic}
\end{algorithm}

\section{Conclusion and Future Directions}

The existing literature of choice modeling in the feature-free and  feature-based settings have been largely separate from each other. In this paper, we propose a neural choice model, a unified neural network framework that applies to both settings through a binary representation of the assortment and feature encoders.
In addition, we provide a MIP formulation for the assortment optimization for a trained neural choice model.

From extensive numerical experiments, we gain insights into the performance of the various methods and summarize into the following three points: First, a single neural choice architecture is capable of recovering all the existing choice models, with a standard learning procedure. The neural choice model becomes particularly effective when the underlying model/training data is too complex to be described by a simple model such as MNL and when there is a sufficient amount of training data.  Second, the neural choice models are robust under the assortment's distribution shift and can serve as a meta choice model. Third, the combined calibration and assortment optimization regime based on the neural network can outperform other optimization heuristics, especially when the underlying model is complex and when we have capacity constraints; the drawback nevertheless is that as it is based on a MIP, the size of the optimization cannot scale to more than around a hundred products; we need further research and new optimization algorithms to scale beyond this.

Moreover, we conclude with the following two future directions:
\begin{itemize}
    \item \textbf{Multiple-purchase choice model.} The existing choice models focus exclusively on modeling the single-choice behavior where only one product is chosen from the offered assortment. The multiple-purchase choice model is usually studied under an inverse optimization framework in the literature of revealed preference \citep{zadimoghaddam2012efficiently,amin2015online,birge2022learning}. The deep learning-based choice models provide a natural framework to study the multiple-choice behavior, and they can fit to such training data by modifying  the loss function.
    \item \textbf{Pricing optimization.} Although we equip our neural networks with an assortment optimization MIP, in a retailing context, the product price is another important factor that affects customer choice. Our current models do not account for this factor but assume prices as fixed.  It is an interesting research problem to study optimal pricing under the neural choice model.
\end{itemize}

\subsection*{Acknowledgment}
The authors thank Antoine Desir for helpful discussions and pointing to the two public datasets for numerical experiments.

\bibliographystyle{informs2014}
\bibliography{main}

\begin{thebibliography}{56}
\expandafter\ifx\csname natexlab\endcsname\relax\def\natexlab#1{#1}\fi
\expandafter\ifx\csname url\endcsname\relax
  \def\url#1{{\tt #1}}\fi
\expandafter\ifx\csname urlprefix\endcsname\relax\def\urlprefix{URL }\fi
\expandafter\ifx\csname urlstyle\endcsname\relax
  \expandafter\ifx\csname doi\endcsname\relax
  \def\doi#1{doi:\discretionary{}{}{}#1}\fi \else
  \expandafter\ifx\csname doi\endcsname\relax
  \def\doi{doi:\discretionary{}{}{}\begingroup \urlstyle{rm}\Url}\fi \fi

\bibitem[{Alptekino{\u{g}}lu and Semple(2016)}]{alptekinouglu2016exponomial}
Alptekino{\u{g}}lu, Ayd{\i}n, John~H Semple. 2016.
\newblock The exponomial choice model: A new alternative for assortment and
  price optimization.
\newblock {\it Operations Research\/} {\bf 64}(1) 79--93.

\bibitem[{Amin et~al.(2015)Amin, Cummings, Dworkin, Kearns, and
  Roth}]{amin2015online}
Amin, Kareem, Rachel Cummings, Lili Dworkin, Michael Kearns, Aaron Roth. 2015.
\newblock Online learning and profit maximization from revealed preferences.
\newblock {\it Proceedings of the AAAI Conference on Artificial
  Intelligence\/}, vol.~29.

\bibitem[{Aouad and D{\'e}sir(2022)}]{aouad2022representing}
Aouad, Ali, Antoine D{\'e}sir. 2022.
\newblock Representing random utility choice models with neural networks.
\newblock {\it arXiv preprint arXiv:2207.12877\/} .

\bibitem[{Ariely and Jones(2008)}]{ariely2008predictably}
Ariely, Dan, Simon Jones. 2008.
\newblock {\it Predictably irrational\/}.
\newblock HarperCollins New York.

\bibitem[{Arkoudi et~al.(2023)Arkoudi, Krueger, Azevedo, and
  Pereira}]{arkoudi2023combining}
Arkoudi, Ioanna, Rico Krueger, Carlos~Lima Azevedo, Francisco~C Pereira. 2023.
\newblock Combining discrete choice models and neural networks through
  embeddings: Formulation, interpretability and performance.
\newblock {\it Transportation Research Part B: Methodological\/} {\bf 175}
  102783.

\bibitem[{Banerjee et~al.(2020)Banerjee, Gupta, Vyas, Mishra
  et~al.}]{banerjee2020exploring}
Banerjee, Kunal, Rishi~Raj Gupta, Karthik Vyas, Biswajit Mishra, et~al. 2020.
\newblock Exploring alternatives to softmax function.
\newblock {\it arXiv preprint arXiv:2011.11538\/} .

\bibitem[{Barratt and Boyd(2022)}]{article}
Barratt, Shane, Stephen Boyd. 2022.
\newblock Fitting feature-dependent markov chains.
\newblock {\it Journal of Global Optimization\/}
  1--18\doi{10.1007/s10898-022-01198-0}.

\bibitem[{Belotti et~al.(2016)Belotti, Bonami, Fischetti, Lodi, Monaci,
  Nogales-G{\'o}mez, and Salvagnin}]{belotti2016handling}
Belotti, Pietro, Pierre Bonami, Matteo Fischetti, Andrea Lodi, Michele Monaci,
  Amaya Nogales-G{\'o}mez, Domenico Salvagnin. 2016.
\newblock On handling indicator constraints in mixed integer programming.
\newblock {\it Computational Optimization and Applications\/} {\bf 65}(3)
  545--566.

\bibitem[{Ben-Akiva and Lerman(1985)}]{Ben-Akiva85}
Ben-Akiva, M., S.~Lerman. 1985.
\newblock {\it Discrete-Choice Analysis: {T}heory and Application to Travel
  Demand\/}.
\newblock {MIT} Press, Cambridge, MA.

\bibitem[{Bentz and Merunka(2000)}]{bentz2000neural}
Bentz, Yves, Dwight Merunka. 2000.
\newblock Neural networks and the multinomial logit for brand choice modelling:
  a hybrid approach.
\newblock {\it Journal of Forecasting\/} {\bf 19}(3) 177--200.

\bibitem[{Berbeglia(2016)}]{berbeglia2016discrete}
Berbeglia, Gerardo. 2016.
\newblock Discrete choice models based on random walks.
\newblock {\it Operations Research Letters\/} {\bf 44}(2) 234--237.

\bibitem[{Bierlaire(2018)}]{SwissMetro}
Bierlaire, M. 2018 \urlprefix\url{http://transp-or.epfl.ch/documents/
  technicalReports/CS_SwissmetroDescription.pdf}.

\bibitem[{Birge et~al.(2022)Birge, Li, and Sun}]{birge2022learning}
Birge, John~R, Xiaocheng Li, Chunlin Sun. 2022.
\newblock Learning from stochastically revealed preference.
\newblock {\it arXiv preprint arXiv:2206.01484\/} .

\bibitem[{Blanchet et~al.(2016)Blanchet, Gallego, and
  Goyal}]{blanchet2016markov}
Blanchet, Jose, Guillermo Gallego, Vineet Goyal. 2016.
\newblock A markov chain approximation to choice modeling.
\newblock {\it Operations Research\/} {\bf 64}(4) 886--905.

\bibitem[{Block and Marschak(1959)}]{block1959random}
Block, Henry~David, Jacob Marschak. 1959.
\newblock Random orderings and stochastic theories of response .

\bibitem[{Bodea et~al.(2009)Bodea, Ferguson, and Garrow}]{bodea2009data}
Bodea, Tudor, Mark Ferguson, Laurie Garrow. 2009.
\newblock Data set—choice-based revenue management: Data from a major hotel
  chain.
\newblock {\it Manufacturing \& Service Operations Management\/} {\bf 11}(2)
  356--361.

\bibitem[{Bront et~al.(2009)Bront, M{\'e}ndez-D{\'\i}az, and
  Vulcano}]{bront2009column}
Bront, Juan Jos{\'e}~Miranda, Isabel M{\'e}ndez-D{\'\i}az, Gustavo Vulcano.
  2009.
\newblock A column generation algorithm for choice-based network revenue
  management.
\newblock {\it Operations research\/} {\bf 57}(3) 769--784.

\bibitem[{Chen et~al.(2021)Chen, Gallego, and Tang}]{chen2021estimating}
Chen, Ningyuan, Guillermo Gallego, Zhuodong Tang. 2021.
\newblock Estimating discrete choice models with random forests.
\newblock {\it INFORMS International Conference on Service Science\/}.
  Springer, 184--196.

\bibitem[{Chen and Mi{\v{s}}i{\'c}(2022)}]{chen2022decision}
Chen, Yi-Chun, Velibor~V Mi{\v{s}}i{\'c}. 2022.
\newblock Decision forest: A nonparametric approach to modeling irrational
  choice.
\newblock {\it Management Science\/} .

\bibitem[{Conforti et~al.(2014)Conforti, Cornu{\'e}jols, Zambelli
  et~al.}]{conforti2014integer}
Conforti, Michele, G{\'e}rard Cornu{\'e}jols, Giacomo Zambelli, et~al. 2014.
\newblock {\it Integer programming\/}, vol. 271.
\newblock Springer.

\bibitem[{D{\'e}sir et~al.(2020)D{\'e}sir, Goyal, Segev, and
  Ye}]{desir2020constrained}
D{\'e}sir, Antoine, Vineet Goyal, Danny Segev, Chun Ye. 2020.
\newblock Constrained assortment optimization under the markov chain--based
  choice model.
\newblock {\it Management Science\/} {\bf 66}(2) 698--721.

\bibitem[{Farias et~al.(2009)Farias, Jagabathula, and Shah}]{farias2009data}
Farias, Vivek, Srikanth Jagabathula, Devavrat Shah. 2009.
\newblock A data-driven approach to modeling choice.
\newblock {\it Advances in Neural Information Processing Systems\/} {\bf 22}.

\bibitem[{Farias et~al.(2013)Farias, Jagabathula, and
  Shah}]{farias2013nonparametric}
Farias, Vivek~F, Srikanth Jagabathula, Devavrat Shah. 2013.
\newblock A nonparametric approach to modeling choice with limited data.
\newblock {\it Management science\/} {\bf 59}(2) 305--322.

\bibitem[{Feldman and Topaloglu(2017)}]{feldman2017revenue}
Feldman, Jacob~B, Huseyin Topaloglu. 2017.
\newblock Revenue management under the markov chain choice model.
\newblock {\it Operations Research\/} {\bf 65}(5) 1322--1342.

\bibitem[{Fischetti and Jo(2017)}]{fischetti2017deep}
Fischetti, Matteo, Jason Jo. 2017.
\newblock Deep neural networks as 0-1 mixed integer linear programs: A
  feasibility study.
\newblock {\it arXiv preprint arXiv:1712.06174\/} .

\bibitem[{Gabel and Timoshenko(2022)}]{gabel2022product}
Gabel, Sebastian, Artem Timoshenko. 2022.
\newblock Product choice with large assortments: A scalable deep-learning
  model.
\newblock {\it Management Science\/} {\bf 68}(3) 1808--1827.

\bibitem[{Gallego and Topaloglu(2019)}]{gallego2019revenue}
Gallego, Guillermo, Huseyin Topaloglu. 2019.
\newblock {\it Revenue management and pricing analytics\/}, vol. 209.
\newblock Springer.

\bibitem[{Golowich et~al.(2018)Golowich, Rakhlin, and
  Shamir}]{golowich2018size}
Golowich, Noah, Alexander Rakhlin, Ohad Shamir. 2018.
\newblock Size-independent sample complexity of neural networks.
\newblock {\it Conference On Learning Theory\/}. PMLR, 297--299.

\bibitem[{Guo et~al.(2017)Guo, Pleiss, Sun, and
  Weinberger}]{guo2017calibration}
Guo, Chuan, Geoff Pleiss, Yu~Sun, Kilian~Q Weinberger. 2017.
\newblock On calibration of modern neural networks.
\newblock {\it International conference on machine learning\/}. PMLR,
  1321--1330.

\bibitem[{Gupta and Hsu(2020)}]{gupta2020parameter}
Gupta, Arushi, Daniel Hsu. 2020.
\newblock Parameter identification in markov chain choice models.
\newblock {\it Theoretical Computer Science\/} {\bf 808} 99--107.

\bibitem[{Han et~al.(2020)Han, Zegras, Pereira, and Ben-Akiva}]{han2020neural}
Han, Yafei, Christopher Zegras, Francisco~Camara Pereira, Moshe Ben-Akiva.
  2020.
\newblock A neural-embedded choice model: Tastenet-mnl modeling taste
  heterogeneity with flexibility and interpretability.
\newblock {\it arXiv preprint arXiv:2002.00922\/} .

\bibitem[{He et~al.(2016)He, Zhang, Ren, and Sun}]{he2016deep}
He, Kaiming, Xiangyu Zhang, Shaoqing Ren, Jian Sun. 2016.
\newblock Deep residual learning for image recognition.
\newblock {\it Proceedings of the IEEE conference on computer vision and
  pattern recognition\/}. 770--778.

\bibitem[{Hu et~al.(2022)Hu, Simchi-Levi, and Yan}]{hu2022learning}
Hu, Yiqun, David Simchi-Levi, Zhenzhen Yan. 2022.
\newblock Learning mixed multinomial logits with provable guarantees.
\newblock {\it Advances in Neural Information Processing Systems\/} {\bf 35}
  9447--9459.

\bibitem[{Jagabathula(2014)}]{jagabathula2014assortment}
Jagabathula, Srikanth. 2014.
\newblock Assortment optimization under general choice.
\newblock {\it Available at SSRN 2512831\/} .

\bibitem[{Jagabathula et~al.(2020)Jagabathula, Subramanian, and
  Venkataraman}]{jagabathula2020conditional}
Jagabathula, Srikanth, Lakshminarayanan Subramanian, Ashwin Venkataraman. 2020.
\newblock A conditional gradient approach for nonparametric estimation of
  mixing distributions.
\newblock {\it Management Science\/} {\bf 66}(8) 3635--3656.

\bibitem[{Keskar et~al.(2016)Keskar, Mudigere, Nocedal, Smelyanskiy, and
  Tang}]{keskar2016large}
Keskar, Nitish~Shirish, Dheevatsa Mudigere, Jorge Nocedal, Mikhail Smelyanskiy,
  Ping Tak~Peter Tang. 2016.
\newblock On large-batch training for deep learning: Generalization gap and
  sharp minima.
\newblock {\it arXiv preprint arXiv:1609.04836\/} .

\bibitem[{Maurer(2016)}]{maurer2016vector}
Maurer, Andreas. 2016.
\newblock A vector-contraction inequality for {R}ademacher complexities.
\newblock {\it International Conference on Algorithmic Learning Theory\/}.
  Springer, 3--17.

\bibitem[{McFadden and Train(2000)}]{mcfadden2000mixed}
McFadden, Daniel, Kenneth Train. 2000.
\newblock Mixed mnl models for discrete response.
\newblock {\it Journal of applied Econometrics\/} {\bf 15}(5) 447--470.

\bibitem[{M{\'e}ndez-D{\'\i}az et~al.(2014)M{\'e}ndez-D{\'\i}az, Miranda-Bront,
  Vulcano, and Zabala}]{mendez2014branch}
M{\'e}ndez-D{\'\i}az, Isabel, Juan~Jos{\'e} Miranda-Bront, Gustavo Vulcano,
  Paula Zabala. 2014.
\newblock A branch-and-cut algorithm for the latent-class logit assortment
  problem.
\newblock {\it Discrete Applied Mathematics\/} {\bf 164} 246--263.

\bibitem[{Neyshabur et~al.(2015)Neyshabur, Tomioka, and
  Srebro}]{neyshabur2015norm}
Neyshabur, Behnam, Ryota Tomioka, Nathan Srebro. 2015.
\newblock Norm-based capacity control in neural networks.
\newblock {\it Conference on Learning Theory\/}. PMLR, 1376--1401.

\bibitem[{Nixon et~al.(2019)Nixon, Dusenberry, Zhang, Jerfel, and
  Tran}]{nixon2019measuring}
Nixon, Jeremy, Michael~W Dusenberry, Linchuan Zhang, Ghassen Jerfel, Dustin
  Tran. 2019.
\newblock Measuring calibration in deep learning.
\newblock {\it CVPR Workshops\/}, vol.~2.

\bibitem[{Rusmevichientong et~al.(2014)Rusmevichientong, Shmoys, Tong, and
  Topaloglu}]{rusmevichientong2014assortment}
Rusmevichientong, Paat, David Shmoys, Chaoxu Tong, Huseyin Topaloglu. 2014.
\newblock Assortment optimization under the multinomial logit model with random
  choice parameters.
\newblock {\it Production and Operations Management\/} {\bf 23}(11) 2023--2039.

\bibitem[{Rusmevichientong and Topaloglu(2012)}]{rusmevichientong2012robust}
Rusmevichientong, Paat, Huseyin Topaloglu. 2012.
\newblock Robust assortment optimization in revenue management under the
  multinomial logit choice model.
\newblock {\it Operations research\/} {\bf 60}(4) 865--882.

\bibitem[{Seshadri et~al.(2019)Seshadri, Peysakhovich, and
  Ugander}]{seshadri2019discovering}
Seshadri, Arjun, Alex Peysakhovich, Johan Ugander. 2019.
\newblock Discovering context effects from raw choice data.
\newblock {\it International Conference on Machine Learning\/}. PMLR,
  5660--5669.

\bibitem[{Shalev-Shwartz and Ben-David(2014)}]{shalev2014understanding}
Shalev-Shwartz, Shai, Shai Ben-David. 2014.
\newblock {\it Understanding machine learning: From theory to algorithms\/}.
\newblock Cambridge university press.

\bibitem[{Sifringer et~al.(2020)Sifringer, Lurkin, and
  Alahi}]{sifringer2020enhancing}
Sifringer, Brian, Virginie Lurkin, Alexandre Alahi. 2020.
\newblock Enhancing discrete choice models with representation learning.
\newblock {\it Transportation Research Part B: Methodological\/} {\bf 140}
  236--261.

\bibitem[{{\c{S}}im{\c{s}}ek and Topaloglu(2018)}]{csimcsek2018expectation}
{\c{S}}im{\c{s}}ek, A~Serdar, Huseyin Topaloglu. 2018.
\newblock An expectation-maximization algorithm to estimate the parameters of
  the markov chain choice model.
\newblock {\it Operations Research\/} {\bf 66}(3) 748--760.

\bibitem[{Talluri and Van~Ryzin(2004)}]{talluri2004revenue}
Talluri, Kalyan, Garrett Van~Ryzin. 2004.
\newblock Revenue management under a general discrete choice model of consumer
  behavior.
\newblock {\it Management Science\/} {\bf 50}(1) 15--33.

\bibitem[{Talluri et~al.(2004)Talluri, Van~Ryzin, and
  Van~Ryzin}]{talluri2004theory}
Talluri, Kalyan~T, Garrett Van~Ryzin, Garrett Van~Ryzin. 2004.
\newblock {\it The theory and practice of revenue management\/}, vol.~1.
\newblock Springer.

\bibitem[{Tversky and Kahneman(1985)}]{tversky1985framing}
Tversky, Amos, Daniel Kahneman. 1985.
\newblock The framing of decisions and the psychology of choice.
\newblock {\it Behavioral decision making\/}. Springer, 25--41.

\bibitem[{Udwani(2023)}]{udwani2021submodular}
Udwani, Rajan. 2023.
\newblock Submodular order functions and assortment optimization.
\newblock {\it International Conference on Machine Learning\/}. PMLR,
  34584--34614.

\bibitem[{Van~Cranenburgh et~al.(2022)Van~Cranenburgh, Wang, Vij, Pereira, and
  Walker}]{van2022choice}
Van~Cranenburgh, Sander, Shenhao Wang, Akshay Vij, Francisco Pereira, Joan
  Walker. 2022.
\newblock Choice modelling in the age of machine learning-discussion paper.
\newblock {\it Journal of Choice Modelling\/} {\bf 42} 100340.

\bibitem[{Wan et~al.(2013)Wan, Zeiler, Zhang, Le~Cun, and
  Fergus}]{wan2013regularization}
Wan, Li, Matthew Zeiler, Sixin Zhang, Yann Le~Cun, Rob Fergus. 2013.
\newblock Regularization of neural networks using dropconnect.
\newblock {\it International conference on machine learning\/}. PMLR,
  1058--1066.

\bibitem[{Wang et~al.(2020)Wang, Mo, and Zhao}]{wang2020deep}
Wang, Shenhao, Baichuan Mo, Jinhua Zhao. 2020.
\newblock Deep neural networks for choice analysis: Architecture design with
  alternative-specific utility functions.
\newblock {\it Transportation Research Part C: Emerging Technologies\/} {\bf
  112} 234--251.

\bibitem[{Wong and Farooq(2021)}]{wong2021reslogit}
Wong, Melvin, Bilal Farooq. 2021.
\newblock Reslogit: A residual neural network logit model for data-driven
  choice modelling.
\newblock {\it Transportation Research Part C: Emerging Technologies\/} {\bf
  126} 103050.

\bibitem[{Zadimoghaddam and Roth(2012)}]{zadimoghaddam2012efficiently}
Zadimoghaddam, Morteza, Aaron Roth. 2012.
\newblock Efficiently learning from revealed preference.
\newblock {\it International Workshop on Internet and Network Economics\/}.
  Springer, 114--127.

\end{thebibliography}

\appendix

\section{Appendix for Numerical Experiments without Features}
\label{apnx:no_fea_model}
\subsection{Model Details for Table~\ref{tabRP}}

In this subsection, we describe the experiment setup of Table~\ref{tabRP}.  For all neural networks trained in all experiments in the paper,  we always set the batch size of training as $100$ with learning rate as $0.0005$ and training epoch as $100$.

\paragraph{Data Generation:}\

The assortments (in samples) are generated by first deciding its size with a uniform distribution and then randomly select the items in the assortment. The parameters for different underlying choice models are generated in the following way:
\begin{itemize}
    \item MNL: We sample the mean utility $u_i\in \mathbb{R}$ for each product $i=1,...,n$ i.i.d. from the standard Normal distribution  $\mathcal{N}(0,1)$.
    \item MCCM: To distinguish the underlying true MCCM from an MNL model (which as we mentioned is a special case of MCCM), we consider the following generation mechanism:

    We first set $\sigma$ to control deviation ($\sigma = 2.5$ for dataset $n=20$ and $\sigma = 4$ for dataset $n=50$). We divide products into groups with a cluster number $c_{\text{num}}$ ($c_{\text{num}} = 4$ for dataset $n=20$ and $c_{\text{num}} = 10$ for dataset $n=50$) to denote the number of groups. Then arriving probability $\lambda$ is generated by:
    $$
\lambda_i = \dfrac{\exp(\mu_i)}{\sum_{j=1}^{n} \exp(\mu_j)},
$$
where $\mu_i$, $i=1,...,n$ is identically and independently (i.i.d.) sampled from a Normal distribution $\mathcal{N}(0,\sigma^2)$. And the transition matrix $\rho$ is generated row by row. For all the $n$ rows, each row $\rho_i$ is generated following:

$$\rho_{i,j}=\dfrac{\exp(\nu_{i,j})}{\sum_{j=1}^{n} \exp(\nu_{i,j})},$$
where $\nu_{i,j}$, $i,j=1,...,n$ is independently sampled by Normal distribution $\mathcal{N}(\bar{\nu}_{i,j},\sigma^2)$ and
\begin{equation*}
\bar{\nu}_{i,j} =
    \begin{cases}
      2\sigma & \forall (i,j)\in \left\{(i,j)|\exists k\in \{0,...,c_{\text{num}}-1\}, kn/c_{\text{num}}< i,j\leq  (k+1)n/c_{\text{num}} \right\}, \\
      0  & \text{otherwise}.
    \end{cases}
\end{equation*}
The intuition for the above is the following: We assume the set of all products can be divided evenly into several groups, and when a wanted product turns out to be unavailable, the customer is most likely to choose an alternative within the same group of this product.

    \item NP: The non-parametric (NP) choice model assumes that there is a fixed number of alternating schemes/types amongst customers. Each alternating scheme denotes a permutation of the total set of products, and following this scheme customers will check from the first product to the last until they find a product that is in the assortment. The formal definition is introduced in Section~\ref{subsec:RUM}. Let $n_{\text{perm}}$ denote the number of candidate permutations with positive probability, we set $n_{\text{perm}} = 10$ for dataset $n=20$ and $n_{\text{perm}}=20$ for dataset $n=50$, and the parameters for the underlying model are the $n_{\text{perm}}$ permutations of $\{1,...,n\}$ and their corresponding probabilities. These parameters are randomly uniformly generated, and following the ground truth non-parametric model we generate synthetic samples.
    \item MMNL: We assume there are $5$ customer types with equal sampled probability, i.e., $\alpha_c=0.2$ for all $c=1,...,5$. For each type $c$, we set the mean utility $u_{c,i}\in \mathbb{R}$ for each product $i=1,...,n$ as:
\begin{equation*}
u_{c,i} \sim
    \begin{cases}
      \mathcal{N}(c+\frac{n}{5},1) & (c-1)n/5+1\leq i\leq cn/5,\\
      0 & i=0,\\
      -50  & \text{otherwise}.
    \end{cases}
\end{equation*}
Thus, each customer type $c$ will (almost) only consider a fixed set of products belonging to $(c-1)n/5+1\leq i\leq cn/5$, and we set no-purchase's mean utility as $0$ for all types.
\end{itemize}

\subsection{Benchmark Algorithms}

\begin{itemize}
     \item MCCM-EM: We follow \citet{csimcsek2018expectation}'s EM algorithm to estimate the parameters of the MCCM. In summary, the algorithm starts with some random initial transition matrix $\rho$ and arrival probabilities $\lambda$. For the expectation step, and based on the current transition matrix and arrival probabilities, the algorithm computes the expected probability of being interested in each product when a customer arrives and the number of times that a customer transitions from one product to another during the decision process for each sample. For the maximization step, it maximizes the likelihood to update the transition matrix and arrival probabilities by treating the previous two expected values as true values. The algorithm repeats above two steps until it converges. In our experiments, the initial transition matrix and arrival probabilities are chosen by uniformly randomizing each entry and normalizing them to stochastic matrix or distribution vector. We terminate the algorithm when the mean changes of updating transition matrix and arrival probabilities is smaller than $0.01$.
\item MNL-MLE: We assume the underlying choice model is the MNL model and ignore all observed features (if any). We estimate the unknown mean utilities of products by maximizing the likelihood (MLE).
\end{itemize}
\section{Appendix for Numerical Experiments with Features}
\label{sec:apx_fea_model}

\subsection{More Details for Table~\ref{tabFeat1}}

\paragraph{Feature-based Choice Models:}\

Given an assortment $\mathcal{S}\subset\mathcal{N}$ and an input feature vector $\bm{z}_i\in \mathbb{R}^d$ for product $i\in\mathcal{N}$ (which can be product features along or the concatenation of customer features and product features):
\begin{itemize}
    \item \textbf{Feature Based MNL Choice Model:} The mean utility $u_i(\bm{z}_i)$ for product $i\in \mathcal{N}$ is
    $$u_i(\bm{z}_i)=\bm{z}_i^\top \bm{\beta},$$
    where $\bm{\beta} \in \mathbb{R}^d$ is the coefficients of features. 
    \item \textbf{Feature Based Markov Chain Choice Model:} This model is from \cite{article}. For product $i\in \mathcal{N}$, the arriving probability $\lambda_i$
      $$\lambda_i(\bm{z}_i) = \frac{\exp(\bm{z}_i^\top \bm{\beta} )}{\sum_{j=1}^n\exp( \bm{z}_j^\top \bm{\beta}  )},$$
    where $\bm{\beta} \in \mathbb{R}^d$. And the vector $\rho_i$ in transition matrix
    $$\rho_{i,j}(\bm{z}_i) = \frac{\exp(\bm{A}_j\bm{z}_{i})}{\sum_{j=1}^n\exp(\bm{A}_{j}\bm{z}_i)},$$
    where $\bm{A}\in \mathbb{R}^{n\times d}$ and $\bm{A}_j$ is $j$-th row of $A$. Then the choice probabilities can be computed  as in MCCM based on $\lambda$'s and $\rho$'s.
\end{itemize}

\paragraph{Data Generation:} \

 For the features, we generate the $j$-th feature $z_{i,j}$ of product $i$ from  i.i.d. standard Normal distribution $\mathcal{N}(0,1)$. For parameters $\bm{\beta}$ in feature-based MNL and ($\bm{A}, \bm{\beta}$) in feature-based MCCM, we generate their entries also from  i.i.d. standard Normal distribution $\mathcal{N}(0,1)$. The assortments (in samples) are generated by first deciding its size with a uniform distribution and then randomly select the items in the assortment.

\paragraph{Architecture Tuning:}\

\begin{table*}[ht!]
\centering
\begin{tabular}{c|c|c|c|c}

\midrule
 MCCM &Gated configuration & CE loss  & Res configuration & CE loss\\
\midrule
\multirow{2}{*}{Benchmark}
& Oracle & 1.932& Oracle & 1.932 \\
& MNL-MLE(f) & 2.212& MNL-MLE(f) & 2.212 \\
\midrule
\multirow{4}{*}{Assort-Net}
& [50,50] & 2.010 & [50,50] & 1.988\\
& [50] & 2.017 & [50] & 1.997\\
& [50,50,50] & 2.015 & [50,50,50] & 2.000\\
& [100,50] & 2.009 & [100, 50] & 2.008\\
\midrule
\multirow{10}{*}{Assort-Net(f)}
& [[1], [50,50]] & 2.010 & [[1], [100,50]] & 1.984\\
& [[5,1], [50,50]] & 2.010 & [[5,1], [100,50]] & 1.985\\
& [[5,5,1], [50,50]] & 2.008 & [[5,5,1], [100,50]] & 1.985\\
& [[5,5,5,1], [50,50]] & 2.007 & [[5,5,5,1], [100,50]] & 1.984\\
& [[10,1], [50,50]] & 2.008 & [[10,1], [100,50]] & 1.984 \\
& [[15,1], [50,50]] & 2.009 & [[15,1], [100,50]] & 1.986\\
& [[20,1], [50,50]] & 2.008 & [[20,1], [100,50]] & 1.983\\
& [[5,1], [50]] & 2.017 & [[5,1], [50]] & 2.018\\
& [[5,1], [50,50,50]] & 2.019 & [[5,1], [100,100,50]] & 1.987\\
& [[5,1], [100,50]] & 2.010 & [[5,1], [100,50]] & 1.985\\

\bottomrule
    \end{tabular}
    \caption{Extension of Table~\ref{tabFeat1} with more architectures. The data is generated by feature-based MCCM.}
    \label{tabFeat}
\end{table*}


Table~\ref{tabFeat} extends Table~\ref{tabFeat1} with more neural network architectures. We find the model performance is quite stable with respect to architecture tuning.


We use Linear$(a, b)$ as an affine function with ReLu as the activation function, there the input dimension is $a$ and the output dimension is $b$. For neural networks without features (under ``Assort-Net''), the model configurations for Gated-Assort-Net, $[50]$ denotes a one layer Linear$(50, 50)$, and $[50, 50]$ denotes a two-layer network, Linear$(50, 50)$ followed by a Linear$(50, 50)$. $[50, 50, 50]$ denotes a three-layer network, and $[100, 50]$ denotes a two-layer, Linear$(50, 100)$ followed by a Linear$(100, 50)$. For the Res-Assort-Net part, $[50]$ denotes one residual block, $[50, 50]$ denotes two residual blocks, $[50, 50, 50]$ three residual blocks, and $[100, 50]$ denotes a Res-Assort-Net with one residual block but contains a hidden layer of size $100$ (so each block contains more than one layers, which is different and extended from Figure~\ref{fig:net_no_feat}(b)).

For neural networks with features (under ``Assort-Net(f)''),  in the ``Gated configuration'' and ``Res configuration'' columns, the first list ($[1]$ or $[5,1]$ or $[5,5,1]$ etc.) denotes the structure of the product encoder, and the second list ($[50]$ or $[50,50]$ etc.) denotes the structure of the assortment network. The meanings of the configuration lists for the assortment network are generally the same as the feature-free cases, though some additional explanations need to be given for the Res-Assort-Net(f). According the network structure, the assortment network takes as input the concatenated vector of latent utilities and assortments, leading to a vector of length $100$. $[100, 50]$ indicates one residual block with the linear part Linear$(100, 100)$ and an ending layer Linear$(100, 50)$. $[50]$ indicates no residual block and just an ending layer Linear$(100, 50)$. $[100, 100, 50]$ indicates two residual blocks with linear part Linear$(100, 100)$ and an ending layer Linear$(100, 50)$. The first list gives the structure of product encoder. Recall that the length of the product feature $d=5$, and product features are encoded into latent utility values by several layers of fully connected network. $[1]$ denotes a single layer Linear$(5,1)$, $[5,1]$ denotes a Linear$(5,5)$ followed by a Linear$(5,1)$. $[10,1]$ denotes a Linear$(5,10)$ followed by a Linear$(10,1)$. The rest  follow similar rules.


\subsection{More Details on Real Data Experiments}
\label{sec: real_data_exp}
\paragraph{Data Description:}\

\begin{itemize}
    \item \textbf{SwissMetro.}   The SwissMetro Data Set is a public survey data set collected to analyze preference between the Metro system against the other transport modes  car and train.  The description of the dataset is given in \cite{SwissMetro}. For convenience of network training, we clean the data by removing samples whose features are taking abnormal values. To be specific, we only keep samples whose ``CHOICE'' (final choice) is not $0$ (which indicates unknown), ``WHO'' (who pays the ticket) value is not $0$ (which indicates unknown), ``AGE'' is not $6$ (which indicates unknown), ``INCOME'' is not $4$ (which indicates unknown), and ``PURPOSE'' is in $[1,2,3,4]$ (1: Commuter, 2: Shopping, 3: Business, 4: Leisure). After cleaning, there are in total $9,135$ samples and the number of alternatives in assortment varies from $2$ to $3$ as only when owning the car the customer can have that choice. The sample sizes in train, validate and test process are $7,000$, $1,000$ and $1,000$ correspondingly. There are  $8$ customer features and $3$ product features, shown in Table~\ref{tab:Features_real}. For product features. Note that ``CAR'' has no ``HE'' (Headaway, which means interval time between two consecutive train/metro/car arrivals) feature in the original dataset, so we set it to $0$. This is reasonable since for private cars the waiting time is always $0$. In this dataset the product features are not static, as they vary among different types of customers.

\item \textbf{Expedia Search.} The Expedia Data Set is a public data set from a Kaggle competition ``Personalized Expedia Hotel Searches - ICDM2013''. It contains the ordered list of hotels according to the user's search. We clean the data by dropping the searches with outlier prices ($>\$1,000$) and days of booking in advance (gap between booking date and check-in date $>365$) and also the features with missing values. After cleaning, there are in total $386,557$ searches where the size of the offered hotel/searched results list (i.e., the number of alternatives in one assortment) is between $5$ to $38$, and $136,886$ unique hotels. The sample sizes in train, validate and test process are $200,000$, $20,000$ and $20,000$ correspondingly. There are  $6$ customer features and $7$ product features as shown in Table~\ref{tab:Features_real}.

For assortment vector $\bm{S}$, we encode the ranked positions in assortments ($1-39$ with $39$ as no-purchase option) and if some position has no shown product, we assign their product features as $0$. Thus the assortment vector captures the ranking and assortment structure within the assortment. Notice that we only use the positions of products in offered assortments instead of the products themselves. There are several reasons: The number of unique products is too large ($136,886$ unique hotels in $386,557$ offered assortments after cleaning) to be covered as assortment vector; Also in practice, hotels with different locations would never be included in the same assortment and the assortment effect from such hotels should be ignored; Further,  the neural choice model can utilize the product features (with customer features) to identify hotels' utilities and also explore the assortment effect from the different ranked positions thanks to the architecture.
\end{itemize}
\begin{table}[ht!]
    \centering
    \begin{tabular}{c|c|c}
    \toprule
           Data&Customer Feature&Product Feature \\
\midrule
 \multirow{8}{*}{SwissMetro}
 &Gender& Travel time\\
 &Age&Headaway\\
 &Income&Cost\\
 &First class or not&\\
 &Who pays (self, employer...)&\\
 &Purpose (commuter, shopping...)&\\
 &\# luggages&\\
 &Owning annual ticket or not&\\
 \midrule
 \multirow{7}{*}{Expedia} &\# nights stay &Hotel star rating\\
 &\# days booking in advance &Hotel chain or not\\
 &\# adults& Location score (by Expedia)\\
 &\# children& Historical prices\\
 &\# rooms &Price\\
 & Saturday included or not&Promotion or not\\
 &&Randomly ranked or not (by Expedia)\\
 \bottomrule
    \end{tabular}
    \caption{Features in SwissMetro and Expedia.}
    \label{tab:Features_real}
\end{table}

\paragraph{Tuning Architectures for Neural Networks and Performance:}\

The performances of different models on the SwissMetro and Expedia datasets are summarized in Table~\ref{tab:RD}, which is extended from Table~\ref{tabExp}. The configurations of both our assort-nets and the benchmark nets are given in the form of channel arrays. For our assort-nets, the first array denotes the structure of customer encoder, the second array denotes the product encoder, and the third denotes the assort-net part of the whole network. For instance, for the gated-assort-net in the configuration list $[[30,10],[10,10],[30,30,3/39]]$, the first array $[30, 10]$ denotes the structure of customer encoder, having one hidden layer of size $30$ and a final layer of size $10$. The second array $[10, 10]$ denotes the structure of the product encoder, having one hidden layer of size $10$ and a final layer of size $10$. The size of the final layer of the customer encoder equals that of the product encoder. The third array $[30, 30, 3/39]$ denotes that the assort-net part has two hidden layers of size $30$. The size of the final layer is dependent on the number of total products $n$, which is $3$ for SwissMetro and $39$ for Expedia. The Res-Assort-Net follows the same way to denote the configurations of customer encoder and product encoder. To denote the configuration of the Res-Assort-Net part, take the last configuration $[6/78, 60, 6/78]*2$ for example, it denotes that there are 2 residual blocks, each containing a hidden layer of size $60$ and the input and output of one block both have dimensions $6/78$. Note there is a final (fully-connected) layer with input dimension $6/78$ and output dimension $3/39$
 
We defer the description of benchmark algorithms in Appendix~\ref{sec:bench_alg}. The configurations ($[100, 9]$ or $[100, 30]$) denote the corresponding network structure, having a hidden layer of size $100$ and a final layer of size $9$ or $30$.

For Random Forest, the hyper-parameters including number of trees and maximum depth are chosen by validation dataset performance. Specifically, the number of trees is tested among $\{10, 20, 50, 100, 200, 400\}$ and the maximum depth is tested among $\{4, 6, 8, 12, 16, 18, 20\}$. Other parameters are set as default values in the Python package scikit-learn's function \textit{sklearn.ensemble.RandomForestClassifier}.

We restrict all three parts of our network (product encoder, customer encoder and assort-net) to a depth of no more than $2$. The widths are tuned a little bit loosely (due to their relatively lesser effect on the model performance), but no more than $200$. The architecture of TasteNet follows the original paper \cite{han2020neural}, in which they used 110 hidden units. The architecture of DeepMNL is set to keep some degree of similarity with the architecture of TasteNet.

\begin{table*}[ht!]
    \centering
    \begin{tabular}{c|c|c|c}
    \toprule
         Model&Config&SwissMetr&Expedia \\
         \midrule
\multirow{6}{*}{Gated-Assort-Net(f)}
&[[10],[10],[3/39]]&0.696&2.479\\
&[[100,10],[10],[3/39]]&0.638&2.468\\
&[[10],[10,10],[3/39]]&0.673&2.458\\
&[[10],[10],[10,3/39]]&0.708&2.463\\
&[[100,10],[10,10],[10,3/39]]&\textbf{0.598}&2.458\\
&[[30,10],[10,10],[30,30,3/39]]&0.623&2.403\\
\midrule
\multirow{3}{*}{Res-Assort-Net(f)}&
[[100, 10],
[10, 10],
[6/78, 6/78]*1]&0.607&2.468\\
&
[[100, 10],
[10, 10],
[6/78,60,6/78]*1]&0.608&2.355\\
&[[200, 20],
[20, 20],
[6/78,60,6/78]*2]&0.610&\textbf{2.325}\\
\midrule
MNL-MLE&-- & 0.883
& 2.827 \\
\midrule
MNL(f)-MLE&-- &0.810

&2.407\\
\midrule
\multirow{2}{*}{TasteNet}
&[100, 9]
 +MNL
&0.691&--\\
&[100, 9]
 +MNL(f)
&0.681&2.436\\
\midrule
\multirow{2}{*}{DeepMNL} &[100, 30]
 +MNL
&0.751&--\\
&[100, 30]
 +MNL(f)
&0.741&2.374\\
\midrule
Random Forest &--&0.633&2.458\\
   \bottomrule
    \end{tabular}
    \caption{Performance on the SwissMetro and Expedia (extended from Table~\ref{tabExp}).}
    \label{tab:RD}
\end{table*}

\subsection{Benchmark Algorithms}
\label{sec:bench_alg}
\begin{itemize}
\item MNL(f)-MLE: We assume the underlying choice model is the feature-based MNL choice model introduced in Appendix~\ref{sec:apx_fea_model}, where the $i$'s feature vector $\bm{z}_i=(\bm{f}^C,\bm{f}^P_i)$ is a concatenation of customer feature $\bm{f}^C$ (which may vary in different samples) and product feature $\bm{f}^P_i$. We estimate the unknown parameters by MLE.

\item TasteNet: We follow the architecture introduced by \cite{han2020neural}. We train a neural network $g_{\theta}(\cdot)$ to encode the customer feature vector $\bm{f}^C\in \mathbb{R}^{d'}$ into a vector $(\bm{\alpha}_1,...,\bm{\alpha}_n) \coloneqq g_{\theta}(\bm{f}^C)\in \mathbb{R}^{nd}$, where $\bm{\alpha}_i\in \mathbb{R}^d$ is interpreted as the coefficient vector of product features $\bm{f}^P_i$. For product $i\in \{1,..., n\}$, its mean utility is encoded by:
    $$
    u_i:=\bm{\alpha}^\top_i\bm{f}^P_i.
    $$
    This utility vector $\bm{u}\in \mathbb{R}^n$ is then passed to a softmax layer (i.e., an MNL operator), giving the predicted choice probability.

    \item DeepMNL: We modify the architecture introduced by \cite{sifringer2020enhancing}.
    We train a neural network $g_{\theta}(\cdot)$ to encode the concatenated feature vector $(\bm{f}^C,\bm{f}^P_i)\in \mathbb{R}^{d+d'} $of customer features $\bm{f}^C\in \mathbb{R}^{d'}$ and product features $\bm{f}^P_i \in \mathbb{R}^{d}$ into scalar $g_{\theta}\left( (\bm{f}^C,\bm{f}^P_i) \right)\in \mathbb{R}$ as the mean utility of product $i$ in MNL model. This utility vector is then passed to a softmax layer, giving the predicted choice probability.
    \item Random Forest: We follow the work \cite{chen2021estimating} to apply the random forest method to predict customer choice. Specifically, for each sample assortment, we concatenate the assortment vector $\bm{S}$, customer feature and all product features as its input (thus same as Gated-Assort-Net(f) and Res-Assort-Net(f)). For the products not offered in assortment, we assign their product feature values as $0$.
\end{itemize}
\paragraph{Dealing with Missing Products for TasteNet and DeepMNL:}\

Due to the fixed size of input vectors when training the neural networks, the DeepMNL and TasteNet cannot be applied directly when the size of assortment changes across different samples. For SwissMetro dataset, TasteNet and DeepMNL are only trained based on the samples with full assortments where all products are offered. The number of such samples in the dataset is $7,839$. We predict the full assortments choice probabilities by using trained TasteNet or DeepMNL, and use trained MNL-MLE or MNL(f)-MLE to predict the other assortments (with missing product) where the sample size is $1,296$ in the dataset. We denote the corresponding prediction models as TasteNet(DeepMNL)+MNL and TasteNet(DeepMNL)+MNL(f); For Expedia dataset, since there are only $9$ out of $386,557$ samples are full assortments, we cannot use the same method as in SwissMetro due to the lacking data (otherwise we are almost reproducing the MNL-MLE or MNL(f)-MLE). Thus, we just assign $0$ into the missing entries of product feature vectors, which is same as Gated-Assort-Net(f) and Res-Assort-Net(f). Since these methods are recovering the missing feature vectors by $0$ and the last layer is softmax, we also name them as TasteNet(DeepMNL)+MNL(f).

\section{Proof for Subsection~\ref{subsec:theorem}}
\label{apnx:proof}
The idea of obtaining a generalization bound for neural networks via Rademacher complexity is not new  \citep{wan2013regularization,neyshabur2015norm,golowich2018size}. We customize the previous analysis to the case of GAsN---the key is the peeling argument (see \cite{neyshabur2015norm} for an example). We first introduce the definition of \textit{empirical Rademacher complexity} which captures the complexity of a function class as a machine learning model. It is called as ``empirical'' because the complexity measure is contingent on the underlying dataset, and thus is a random variable itself.

\begin{definition}[Empirical Rademacher Complexity]
The empirical Rademacher complexity of function class $\mathcal{G}$ with respect to a dataset $\mathcal{D}=\{\bm{x}_1,...,\bm{x}_m\}$ is defined as
$$\widehat{\mathfrak{R}}_{\mathcal{D}}(\mathcal{G})\coloneqq \frac{1}{m}\mathbb{E}\left[\sup_{g\in \mathcal{G}}\sum_{k=1}^m\epsilon_k g(\bm{x}_k)\right]$$
where the expectation is taken with respect to  $\epsilon_k$'s -- a sequence of independent Rademacher random variables with $\mathbb{P}(\epsilon_k=1)=\mathbb{P}(\epsilon_k=-1)=\frac{1}{2}$.
\end{definition}
The derivation of generalization bound of the neural network reduces to the calculation of the corresponding Rademacher complexity.

\begin{lemma}
\label{apx:bound:Gen_bound}
Suppose $|r((i,\bm{S});\bm{\theta})|\leq C$ holds for all $(i,\bm{S})\in \mathcal{D}$ and $\bm{\theta} \in \Theta$ for some $C>0$, and $(i_k,\mathcal{S}_k)\in \mathcal{D}$ is i.i.d. sampled from distribution $\mathcal{P}$. Then the following inequality holds with probability $1-\delta$,
$$R(\hat{\bm{\theta}}) \le R(\bm{\theta}^*) +2\widehat{\mathfrak{R}}_{\mathcal{D}}(\{r(\cdot;\bm{\theta}):\bm{\theta} \in \Theta\})+5C\sqrt{\frac{2\log(8/\delta)}{m}}.$$
\end{lemma}

\begin{proof}
See Theorem 26.5 of \cite{shalev2014understanding}.
\end{proof}

The following lemma presents a basic property of the Rademacher complexity (with respect to the composition of a Lipschitz function).

\begin{lemma}
\label{bound:concen}
Let $\mathcal{X}$ be an arbitrary set and $\mathcal{G}$ be a class of functions $\bm{g}=(g_1,...,g_n): \mathcal{X}\rightarrow \mathbb{R}^n$.
Suppose $h_k: \mathbb{R}^n \rightarrow \mathbb{R}$ has Lipschitz constant $\kappa$ with respect to Euclidean norm for $k=1,...,m$. Then, the following inequality holds for any fixed $(\bm{x}_1,...,\bm{x}_m)
\in \mathcal{X}^m,$
$$\mathbb{E}\left[\sup_{\bm{g
}\in \mathcal{G}}\sum_{k=1}^m\epsilon_kh_k(\bm{g}(\bm{x}_k)) \right]\leq \sqrt{2}\kappa \mathbb{E}\left[\sup_{\bm{g}\in \mathcal{G}}\sum_{k=1}^m\sum_{i=1}^n\epsilon_{k,i}g_i(\bm{x}_k) \right],$$
where $\{\epsilon_k\}_{k=1,...,m}$ and $\{\epsilon_{k,i}\}_{k=1,...,m} \ \text{for} \  i=1,...,n$ are independent Rademacher sequences. Here the expectations on both sides are taken with respect to these Rademacher random variables.
\end{lemma}

\begin{proof}
See Corollary 4 of \cite{maurer2016vector}.
\end{proof}

For convenience, we rewrite the risk function $r((i,\bm{S});\bm{\theta})$ as a function of $L$-th layer's output $\bm{z}_L$:
$$\phi\left((i,\bm{S}),\bm{z}_L\right)\coloneqq -\log \frac{\exp(z_{L,i})}{\sum_{i'\in \mathcal{S}}\exp(z_{L,i'})}=r((i,\bm{S});\bm{\theta})$$

Given a data sample $(i,\bm{S})$, the following lemma says that $\phi\left((i,\bm{S}),\bm{z}_L\right)$ is a Lipschitz function of $\bm{z}_L$.
\begin{lemma}
\label{bound:Lip}
$\phi\left((i,\bm{S}),\bm{z}_L\right)$ has Lipschitz constant $\sqrt{2}$ with respect to $\bm{z}_L$ under the Euclidean norm.
\end{lemma}
\begin{proof}
For simplicity we drop the subscript $L$. Then the $j$-th component of $\phi((i,\bm{S}),\bm{z})$'s gradient with respect to $\bm{z}$ is:
$$
(\nabla\phi((i,\bm{S}),\bm{z}))_j= \begin{cases}
\frac{\exp(z_{j})}{\sum_{j' \in \mathcal{S}} \exp(z_{j'})}-1, & j=i,\\
\frac{\exp(z_{j})}{\sum_{j' \in \mathcal{S}} \exp(z_{j'})},& j\in \mathcal{S}/i,\\
0,& j\notin \mathcal{S}.
\end{cases}$$

Thus
\begin{align*}
&\|\nabla\phi((i,\bm{S}),\bm{z})\|_2^2\\
=&\left(\frac{\exp(z_{i})}{\sum_{i' \in \mathcal{S}} \exp(z_{i'})}-1\right)^2+\sum_{j\in \mathcal{S}/i} \left(\frac{\exp(z_{j})}{\sum_{i' \in \mathcal{S}} \exp(z_{i'})}\right)^2\\
\leq &\frac{\sum_{i' \in S}\exp(z_{i'})^2}{\left(\sum_{i' \in \mathcal{S}} \exp(z_{i'})\right)^2}+1\leq 2.
\end{align*}

By the mean value theorem, the proof is completed.
\end{proof}

Note that for $l=1,...,L$, the $l$-th layer's output $\bm{z}_l$ is a function of $\bm{\theta}$ and $\bm{S}$. We can define function classes as follows
$$\mathcal{G}_{l,i}\coloneqq \left\{\bm{S}\in \{0,1\}^n\rightarrow z_{l,i}(\bm{S};\bm{\theta}): \bm{\theta} \in \Theta \right \}$$
for $l=1,...,L$ and $i=1,...,n_l.$

The following corollary relates the Rademacher complexities after and before the softmax operator in the last layer.

\begin{corollary}
\label{bound:loss_peel}
The following inequality holds,
$$\widehat{\mathfrak{R}}_{\mathcal{D}}(\{r(\cdot;\bm{\theta}):\bm{\theta} \in \Theta\})\leq 2\sum_{i=1}^n\widehat{\mathfrak{R}}_{\mathcal{D}}(\mathcal{G}_{L,i}).$$
\end{corollary}
\begin{proof}
Denote
$$\mathcal{G}_{L}\coloneqq \left\{\bm{S}\in \{0,1\}^n\rightarrow \bm{z}_{L}(\bm{S};\bm{\theta}): \bm{\theta} \in \Theta \right \}.$$

Let $\mathcal{G}=\mathcal{G}_L$, $\bm{g}(\cdot)=\bm{z}_L(\cdot;\bm{\theta})$ and $h_k(\cdot)=\phi((i_k,\bm{S}_k),\cdot)$ in Lemma~\ref{bound:concen}. Together with Lemma~\ref{bound:Lip}, we have the following
$$\widehat{\mathfrak{R}}_{\mathcal{D}}(\{r(\cdot;\bm{\theta}):\bm{\theta} \in \Theta\})\leq \frac{2}{m}\mathbb{E}\left[\sup_{\bm{\theta}\in \Theta}\sum_{i=1}^n \sum_{k=1}^m \epsilon_{k,i}z_{L,i}(\bm{S}_k;\bm{\theta})\right],$$
where expectations are with respect to $\bm{\epsilon}$. Then for any sequence $\{\epsilon_{k,i}\}$,
$$\sup_{\bm{\theta}\in \Theta}\sum_{i=1}^n \sum_{k=1}^m \epsilon_{k,i}z_{L,i}(\bm{S}_k;\bm{\theta})\leq \sum_{i=1}^n \sup_{\bm{\theta}\in \Theta} \sum_{k=1}^m \epsilon_{k,i}z_{L,i}(\bm{S}_k;\bm{\theta}),$$
where $\epsilon_{k,i}$ are i.i.d. Rademacher random variables. Finally by definition of $\widehat{\mathfrak{R}}_{\mathcal{D}}(\mathcal{G}_{L,i})$ we complete the proof.
\end{proof}

Then we can apply the standard analysis of fully connected networks \citep{wan2013regularization,neyshabur2015norm,golowich2018size}:

\begin{lemma}
\label{bound:NN_peel}
For $i=1,...,n$,
$$\widehat{\mathfrak{R}}_{\mathcal{D}}(\mathcal{G}_{L,i})\leq \frac{1}{\sqrt{m}}\left(\bar{b}\cdot\frac{(2\bar{W})^L-1}{2\bar{W}-1}+(2\bar{W})^{L} \sqrt{2\log(2n)}\right).$$
\end{lemma}
\begin{proof}
All expectations below are with respect to $\bm{\epsilon}$. We begin with the last layer:
\begin{align}
\widehat{\mathfrak{R}}_{\mathcal{D}}(\mathcal{G}_{L,i})=&\frac{1}{m}\mathbb{E}\left[\sup_{\bm{\theta}\in \Theta} \sum_{k=1}^m \epsilon_{k}z_{L,i}(\bm{S}_k;\bm{\theta}) \right] \nonumber \\
=&\frac{1}{m}\mathbb{E}\left[\sup_{\bm{\theta}\in \Theta} \sum_{k=1}^m \epsilon_{k}(\bm{w}_{L,i}^\top \bm{z}_{L-1}(\bm{S}_k;\bm{\theta})+b_{L,i})^+ \right]\nonumber \\
\leq& \frac{1}{m}\mathbb{E}\left[\sup_{\bm{\theta}\in \Theta} \sum_{k=1}^m \epsilon_{k}(\bm{w}_{L,i}^\top \bm{z}_{L-1}(\bm{S}_k;\bm{\theta})+b_{L,i}) \right] \nonumber \\
= &\frac{1}{m}\mathbb{E}\left[\sup_{\bm{\theta}\in \Theta} \sum_{k=1}^m \epsilon_{k}\bm{w}_{L,i}^\top \bm{z}_{L-1}(\bm{S}_k;\bm{\theta}) \right]+\frac{1}{m}\mathbb{E}\left[\sup_{|b_{L,i}|\leq \bar{b}} \sum_{k=1}^m \epsilon_{k}b_{L,i}\right] \label{apx:bound:decompose}.
\end{align}
Here the first inequality comes from the contraction property of Rademacher complexity and the fact that ReLU is $1$-Lipschitz (see Lemma 26.9 of \citep{shalev2014understanding}).

For the first term of (\ref{apx:bound:decompose}),
\begin{align}
\sup_{\bm{\theta}\in \Theta} \sum_{k=1}^m \epsilon_{k}\bm{w}_{L,i}^\top \bm{z}_{L-1}(\bm{S}_k;\bm{\theta})     &\leq \sup_{\bm{\theta}\in \Theta} \|\bm{w}_{L,i}\|_1\left\|\sum_{k=1}^m \epsilon_{k} \bm{z}_{L-1}(\bm{S}_k;\bm{\theta})\right\|_{\infty} \nonumber \\
    &\leq \bar{W} \sup_{\bm{\theta}\in \Theta} \left\|\sum_{k=1}^m \epsilon_{k} \bm{z}_{L-1}(\bm{S}_k;\bm{\theta})\right\|_{\infty} \nonumber\\
    &= \bar{W} \sup_{i=1,...n_{L-1},\bm{\theta}\in\Theta} \left|\sum_{k=1}^m \epsilon_{k} z_{L-1,i}(\bm{S}_k;\bm{\theta})\right| \nonumber\\
    &\leq \bar{W} \sup_{\bm{\theta}\in\Theta} \left|\sum_{k=1}^m \epsilon_{k} \bm{z}_{L-1,1}(\bm{S}_k;\bm{\theta})\right| \nonumber\\
    &\leq \bar{W} \sup_{\bm{\theta}\in\Theta}\left[ \sum_{k=1}^m \epsilon_{k} \bm{z}_{L-1,1}(\bm{S}_k;\bm{\theta})\right]+\bar{W} \sup_{\bm{\theta}\in\Theta}\left[ -\sum_{k=1}^m \epsilon_{k} \bm{z}_{L-1,1}(\bm{S}_k;\bm{\theta})\right]\nonumber,
\end{align}
where the first inequality follows the Cauchy–Schwarz inequality, the second inequality comes from $\|\bm{w}_{L,i}\|_1\leq \bar{W}$, the third inequality comes from $\mathcal{G}_{L-1,1}=\mathcal{G}_{L-1,2}=...=\mathcal{G}_{L-1,n_{L-1}}$ and the last inequality is because $\mathcal{G}_{L-1,1}$ includes the zero function.

Then we note that $\bm{\epsilon}$ and $-\bm{\epsilon}$ have same distribution,
\begin{align*}
\mathbb{E}\left[\sup_{\bm{\theta}\in \Theta} \sum_{k=1}^m \epsilon_{k}\bm{w}_{L,i}^\top \bm{z}_{L-1}(\bm{S}_k;\bm{\theta}) \right]    &\leq  \bar{W} \mathbb{E}\left[\sup_{\bm{\theta}\in\Theta} \sum_{k=1}^m \epsilon_{k} \bm{z}_{L-1,1}(\bm{S}_k;\bm{\theta})\right]+\bar{W}\mathbb{E}\left[ \sup_{\bm{\theta}\in\Theta} -\sum_{k=1}^m \epsilon_{k} \bm{z}_{L-1,1}(\bm{S}_k;\bm{\theta})\right]\\
    &=2\bar{W} \mathbb{E}\left[\sup_{\bm{\theta}\in\Theta} \sum_{k=1}^m \epsilon_{k} \bm{z}_{L-1,1}(\bm{S}_k;\bm{\theta})\right]\\
    &=2\bar{W} m\widehat{\mathfrak{R}}_{\mathcal{D}}(\mathcal{G}_{L-1,1})
\end{align*}

For the second term of (\ref{apx:bound:decompose}), we have

\begin{align*}
    \mathbb{E}\left[\sup_{|b_{L,i}|\leq \bar{b}}\sum_{k=1}^m\epsilon_k b_{L,i}\right] & \leq \mathbb{E}\left[\sup_{|b_{L,i}|\leq \bar{b}}\left|\sum_{k=1}^m\epsilon_k\right| |b_{L,i}|\right]\\
    & \leq \bar{b} \mathbb{E}\left[\left|\sum_{k=1}^m\epsilon_k\right|\right]
    =\bar{b} \mathbb{E}\left[\sqrt{\left(\sum_{k=1}^m\epsilon_k\right)^2}\right]\\
   &\leq  \bar{b} \sqrt{\mathbb{E}\left[\left(\sum_{k=1}^m\epsilon_k\right)^2\right]}
   \leq \bar{b} \sqrt{m},
\end{align*}

where the second last inequality is by Jensen's inequality and the last inequality the definition of Rademacher random variables.

Combining these two parts,, we have $\widehat{\mathfrak{R}}_{\mathcal{D}}(\mathcal{G}_{L,i})\leq 2\bar{W}\widehat{\mathfrak{R}}_{\mathcal{D}}(\mathcal{G}_{L-1,1})+\frac{\bar{b}}{\sqrt{m}}$ for $i=1,...,n$. We can repeat this process $L-1$ times. From Lemma 26.11 in \citep{shalev2014understanding} and by noting $\|\bm{S}\|_{\infty}\leq 1$,
$$\widehat{\mathfrak{R}}_{\mathcal{D}}(\mathcal{G}_{1,1})\leq 2\bar{W}\sqrt{\frac{2\log (2n)}{m}}+\frac{\bar{b}}{\sqrt{m}}.$$
Thus we complete the proof.
\end{proof}

\textbf{Proof for Theorem~\ref{thm:gen_bound}}
\begin{proof}
Combining Lemma~\ref{apx:bound:Gen_bound}, Corollary~\ref{bound:loss_peel}, and Lemma~\ref{bound:NN_peel}, the proof is completed.
\end{proof}

\section{Appendix for Assortment Optimization}
\label{apnx:assort_opt}

For each underlying choice model with a fixed number of available products ($|\mathcal{N}|= 20, 40, 60$), we randomly generate $5$ datasets/instances (with different parameters in the generating choice model) and for each dataset, we randomly generate $20$ revenue vectors and constraints (if any). Thus, each underlying choice model has $100$ randomly sampled optimization problems. Each  method has two steps: the estimation step and the optimization step. The estimation step estimates the parameters of the choice model based on the training data and the optimization step  treats such estimated choice model as the true one and (approximately) optimize it. The precise generation methods, estimation methods, and assortment optimization formulations with corresponding used parameters are introduced below. For the optimization step, we use Gurobi version 9.5.2 for solving all MIPs and LPs. Optimizing the assortment with trained Gated-Assort-Net can be slow since the relaxed problem is not convex. As discussed in Section~\ref{sec:OPT},  we tighten the upper bounds of decision variables (the big-M) following the methods from \cite{fischetti2017deep}, use exponential function's second order approximation $1+x+x^2/2$, and set the maximal solving time limit as $300$ seconds.

\paragraph{Data Generation:}\

Each instance has a training dataset with $30,000$ samples for estimation. The offered assortment $\mathcal{S}_k$ is generated by first  randomly choosing assortment size $|\mathcal{S}_k|$ in $\{1,...,n\}$, and then  randomly picking products from $n$ and form the assortment $\mathcal{S}_k$.  The choice $i_k$ is generated by corresponding choice models including MNL, MCCM, NP, and MMNL, with the same generating process as in Appendix~\ref{apnx:no_fea_model} for their parameters.

For the revenue vector, the $i$-th product's revenue is generated uniformly in the interval $[10,50]$ and we always set the revenue of the no-purchase option as $0$.

For the capacity constraint (if appliable), the constraint is defined by $\bm{a}^T\bm{z}_0\leq c$, where $\bm{a}\in \mathbb{R}^n$ has non-negative elements and $c$ is a positive scalar. We call the $i$-th element $a_i$ as the coefficient of $i$-th item and $c$ as the budget. The $i$-th product's coefficient is generated uniformly in $[10,50]$ and we always set the coefficient of the no-purchase option as $0$. The budget $c$ is generated uniformly among $\max\{\frac{\|\bm{a}\|_{1}}{n},\|\bm{a}\|_{\infty}\}$ and $\max\{\frac{4\|\bm{a}\|_{1}}{n},\|\bm{a}\|_{\infty}\}$, where $\|\bm{a}\|_{1}=\sum_{i=1}^n a_i$ and $\|\bm{a}\|_{\infty}=\max_{i=1,...,n} a_i$. Note such budget can make every product available to be offered in the assortment.

\paragraph{Estimation Step:}\
The estimation step includes Neural Network (NN), MNL-MLE (MNL), and  MCCM-EM (MC). For the neural network estimation method, we test one layer (NN(1)) and two layers (NN(2)) Gated-Assort-Nets with width in each layer equal to the number of products. For MNL-MLE (MNL), and  MCCM-EM (MC), they apply the same process introduced in Appendix~\ref{apnx:no_fea_model}.

\paragraph{Optimization Formulation:}\

Below are the methods implemented to solve the (true) assortment optimization problems of each underlying choice model:
\begin{itemize}
\item MNL and MMNL: When there is no constraint,  we use the revenue ordering method \citep{talluri2004revenue} to get the optimal assortment for MNL and use the MILP formulation in \cite{mendez2014branch} for solving the assortment optimization problem for MMNL. We remark that MNL is a special case of MMNL and  we can easily add the corresponding capacity constraint in the above MILP for both models.
\item MCCM: When there is no constraint, the optimal assortment of MCCM can be computed by repeatedly using Bellman operators \citep{blanchet2016markov} or by solving an LP \citep{feldman2017revenue}. The latter LP method can be extended to an MILP when there exists a cardinality constraint \citep{desir2020constrained}.
\item NP: When there is no constraint, the assortment optimization problem can be formulated as a MILP:
\begin{lemma}
For an NP choice model with revenue vector $\bm{\mu}$, distribution vector $\bm{\lambda}$ on permutations $\text{Perm}_{\mathcal{N}}$, and permutation matrix $\bm{P}_j\in \{0,1\}^{n\times n}$ for $j=1,...,|\text{Perm}_{\mathcal{N}}|$, whose element $P_{i,i'}=1$ if $i'$-th item is ranked in position $i$ and otherwise $0$ in $j$-th permutation, the assortment optimization problem can be formulated by following MILP with optimal assortment as solved $\bm{s}$:
\begin{align}
    \max_{\bm{s},\bm{\tilde{s}},\bm{\eta}} \ \ &\sum_{j=1}^{|\text{Perm}_{\mathcal{N}}|}   \lambda_j \bm{\eta}_j^\top  \bm{P}_j \bm{\mu} \nonumber\\
    \text{s.t. }\  & \bm{P}_j \bm{s}=\tilde{\bm{s}}_j, \ \forall j=1,...,|\text{Perm}_{\mathcal{N}}| \nonumber\\
    &\bm{\eta}_j\leq \tilde{\bm{s}}_j, \ \forall j=1,...,|\text{Perm}_{\mathcal{N}}| \label{NPOpt:3}\\
    & \eta_{j,i}\leq 1-\tilde{s}_{j,i'}, \ \forall j=1,...,|\text{Perm}_{\mathcal{N}}|, \ i=1,...,n, \ i'=1,...,i-1 \label{NPOpt:4}\\
    &\bm{s} \in \{0,1\}^n, \ \bm{\eta}_j=(\eta_{j,1},...,\eta_{j,n})\geq 0, \ \forall j=1,...,|\text{Perm}_{\mathcal{N}}|\nonumber
\end{align}
\end{lemma}
\begin{proof}
Given an assortment vector $\bm{s}\in \{0,1\}^n$ whose element equals $1$ means the corresponding product is included in the assortment, equality constraint sets $\tilde{\bm{s}}_j$ to recover the available positions in permutation $j$: if $i$-th position's item is offered, then $\tilde{s}_{j,i}=1$ otherwise $0$. Constraints (\ref{NPOpt:3}) and (\ref{NPOpt:4}) will together push $\bm{\eta}_j$'s elements into $0$ if the corresponding position is not available (by (\ref{NPOpt:3})) or if any of former positions is available (by (\ref{NPOpt:4})). Thus, since the objective function is maximizing a linear function of $\bm{\eta}_j$ with positive coefficients, $\bm{\eta}_j$ will be a $\{0,1\}^n$ vector with only one element $1$ whose position is the chosen position (the first available position) under assortment vector $\bm{s}$ under $j$-th permutation. By noting $ \bm{P}_j\bm{\mu}$ will return the revenues of each position in $j$-th permutation, the objective function is the expected revenue given assortment $\bm{s}$.
\end{proof}
Note that with additional capacity constraint, we can directly add it into the above formulation to get the optimal assortment.
\end{itemize}

\section{Appendix for Section~\ref{sec:extension}}
\label{apnx:extension}

\subsection{Appendix for Subsection~\ref{subsec:OOD}}
\label{apnx:OOD}
\paragraph{Details for Figure~\ref{fig:EMMC-assort}:}\

 Both the sample generation of the underlying MCCM and EM algorithm implementation are introduced in Appendix~\ref{apnx:no_fea_model}. 
 We compute the mean of absolute differences between estimations and true values as estimation errors, shown on the y-axis. We independently train $10$ times based on different training data: the training data size is always set as $10,000$ with total product number $n=20$, but the assortment size (used in training) is varied and chosen from $2,4,6,8,10$ as shown on the x-axis. The products in each sampled assortment are randomly  chosen. The reported value in the figure is averaged over $10$ independent trails.

 \paragraph{Out-of-domain Performance under Assortment Distribution Shift:}\

In this experiment, we demonstrate whether a different distribution of generating assortments  between training and testing will affect model performance, also known as the out-of-domain generalization. Specifically, we use one distribution to generate the assortment in the training data and test the performance with assortments generated from another distribution. Intuitively, this shows whether the model is capable of generalizing over the domain of assortments.

In the experiment, we generate four groups of training data:
\begin{itemize}
\item For the first dataset D-1, assortments are generated by first  randomly choosing $s=|\mathcal{S}|$,  and then  randomly picking $s$ products from $\mathcal{N}$ to form the assortment $S$.
\item For the second dataset D-2, assortments are generated from a Bernoulli distribution. That is, for each product, with probability $1/2$ it appears in the assortment independent of others.
\item For the third dataset D-3, the assortment is given by first randomly specifying either the first half (product 1 to product $n/2$, suppose $n/2$ is integer for convenience) or the second half (product $n/2+1$ to product $n$) of the total products will definitely not appear in the assortment. Then we generate the assortment within the other half following the same rule as in D-1.
\item For the last dataset D-4, we require all the assortments to have a size of  $n/3$ or $n/3+1$ (again assume they are integers for convenience) products, and then pick products randomly into the assortment. This design is inspired by the results of \cite{gupta2020parameter}.
\end{itemize}

The total number of products $n=30$ and each training dataset consists of $m=100,000$ samples. The underlying true model for generating the data is the MCCM model.

\begin{table}[ht!]
\centering
\begin{tabular}{c|cccc}
\toprule
  & D-1  & D-2  & D-3 & D-4 \\
\midrule
D-1 & 2.364 & 2.603 & 1.820 & 2.298
\\
D-2 & 2.366 & 2.600 & 1.819 & 2.298
\\
D-3 & 2.636	& 2.781 & 1.819 & 2.383
\\
D-4 & 2.434 & 2.619 & 1.829 & 2.297
\\ \midrule
Mix & 2.366 & 2.604 & 1.819 & 2.299 \\
\midrule
Oracle & 2.358 & 2.595 & 1.817 & 2.293
\\
\bottomrule
\end{tabular}
\caption{Out-of-sample CE loss for out-of-domain performance of Gated-Assort-Net. Row headers denote the training dataset and column headers denote the testing dataset. The ``Mix'' header denotes a randomly mixed training set with equal proportion from the $4$ datasets, and the ``Oracle'' header denotes the performance of the true model.}
\label{table:AssortEffect}
\end{table}

We train a neural network model with each of the training data and test its performance on the other datasets. In addition, we train a neural network model with a mixture of training data from the four datasets. Table~\ref{table:AssortEffect} summarizes the out-of-sample CE loss for each training-test pair. For illustration, we train a 1-layer Gated-Assort-Net. The row headers denote the training dataset, and the column headers denote the testing dataset.

The result shows that, except for the dataset D-3 (the third row), Gated-Assort-Net performs well in out-of-domain datasets. It is understandable that the datasets D-1 and D-2 are informative enough for a neural network model to catch the whole picture of the true underlying model. The dataset D-4 also enjoys this privilege, which is a bit surprising but also reinforces the findings in \cite{gupta2020parameter} (though not in an exact same spirit). The distribution of D-3 clearly fails to capture the interplay between the first half and the second half of the total products, and thus cannot reach the performance of the oracle or generalize to different assortment distributions.

\textbf{Details on the true model generation.} The training data is generated from an MCCM. The arriving probability $\lambda_i$ for product $i$ is generated by:
$$
\lambda_i = \dfrac{\exp(\mu_i)}{\sum_{j=1}^{n} \exp(\mu_j)},
$$
where $\mu_i$, $i=1,...,n$ is identically and independently (i.i.d.) sampled by standard Normal distribution $\mathcal{N}(0,1)$. And the transition matrix $\rho$ is generated row by row. For all the $n$ rows, each row $\rho_i$ is generated following the same rule:

$$
\rho_{i,j} = \dfrac{\exp(\nu_{i,j})}{\sum_{j=1}^{n} \exp(\nu_{i,j})},
$$
where $\nu_{i,j}$, ${i,j=1,...,n}$  is i.i.d. sampled by standard Normal distribution $\mathcal{N}(0,1)$.

\subsection{Appendix for Subsection~\ref{subsec:depth}}
\label{apnx:DepWidNN}

In this subsection, we present an experiment on the depth and width of the neural network. Arguably, a wider or deeper neural network gives larger model capacity. However, for the task of choice modeling, we do not recommend a too wide or too deep neural network, mainly for several reasons. First, a complicated neural network architecture requires more number of samples to train, and there might be not enough amount of samples in the application contexts of choice modeling. Second, the customer choice behavior may not need a too complicated neural network to describe. In addition, the assortment optimization's MIP prefers a smaller size one.

Figure~\ref{fig:LenWidNN} presents an experiment on varying the width and depth of the four networks to fit an MCCM with product features. We do not observe a significant change of performance by further widening or deepening the architecture.
\begin{figure}[ht!]
    \centering
    \begin{subfigure}[b]{0.4 \textwidth}
        \centering
        \includegraphics[width=1.\textwidth]{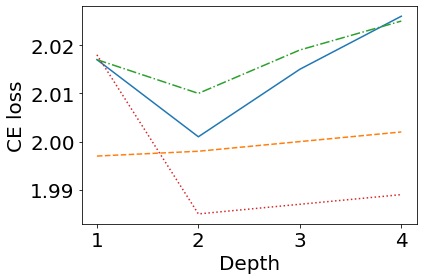}
    \end{subfigure}%
    \begin{subfigure}[b]{0.68 \textwidth}
        \centering
        \includegraphics[width=1.\textwidth]{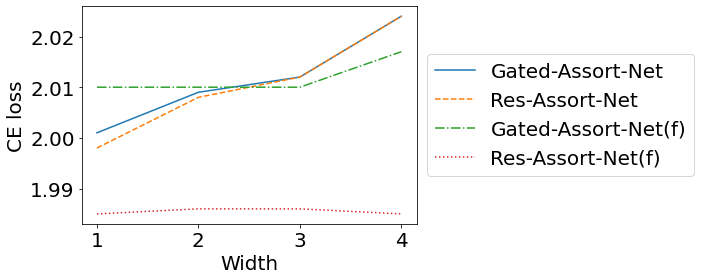}
    \end{subfigure}
    \caption{Network structures and corresponding CE losses.}
    \label{fig:LenWidNN}
\end{figure}

\textbf{Experiment details of Figure~\ref{fig:LenWidNN}:} The underlying model is an MCCM with product features. The sample generation method is the same as that of Table~\ref{tabFeat1}. We set the total number of products $n=50$, and the length of product features $d=5$. For choice networks with features, we use an encoder containing one middle layer to encode each product feature vector to its latent utility. The product encoder is (fixed as) a fully connected layer $\text{Linear}(5,5)$ followed by another fully connected layer $\text{Linear}(5,1)$. And for choice networks without features, we ignore product features and simply take assortments as input.

We carry out our training on $100,000$ samples and the testing set has size $10,000$, and we tune the structure of each neural network model to compare their performances on the testing set.
\begin{itemize}
    \item Gated-Assort-Net: $\text{Depth}=1$ indicates a network with no hidden layer, just a fully connected layer $\text{Linear}(50,50)$ that takes as input the assortment and outputs the final layer. $\text{Depth}=2$ indicates a network with one hidden layer, which is a $\text{Linear}(50,50)$ followed by another $\text{Linear}(50,50)$, and $\text{Depth}=3$ indicates two hidden layers. For the experiment that tunes the width, we fix $\text{Depth}=2$, that is one hidden layer. And we tune the size of the hidden layer. $\text{Width}=2$ indicates a $\text{Linear}(50, 100)$ followed by a $\text{Linear}(100,50)$, and $\text{Width}=3$ indicates a $\text{Linear}(50, 150)$ followed by a $\text{Linear}(150,50)$.
    \item Res-Assort-Net: Depth is exactly the number of residual blocks. For the experiment that tunes the width, we fix that there is one residual block, though this block contains one hidden layer. $\text{Width}=2$ indicates that the fully connected part of the residual block is a
    $\text{Linear}(50, 100)$ followed by a $\text{Linear}(100,50)$, and $\text{Width}=3$ indicates a $\text{Linear}(50, 150)$ followed by a $\text{Linear}(150,50)$.
    \item Gated-Assort-Net(f): Following the same rule introduced above, product features are encoded into a vector of latent utilities. The meanings of depth and width are exactly the same as in Gated-Assort-Net.
    \item Res-Assort-Net(f): According to the structure of the Res-Assort-Net with feature, latent utilities and assortments are concatenated into a vector of length $100$, as the input of the residual blocks. Set aside the ending fully connected layer $\text{Linear}(100, 50)$, the rest parts of the residual blocks are exactly the same as ones in Res-Assort-Net. And so depth is the number of residual blocks, and width indicates the size of the hidden layer. $\text{Width}=2$ indicates that the fully connected part of the residual block is a
    $\text{Linear}(100, 200)$ followed by a $\text{Linear}(200,100)$, and $\text{Width}=3$ indicates a $\text{Linear}(100, 300)$ followed by a $\text{Linear}(300,100)$.
\end{itemize}
\subsection{Appendix for Subsection~\ref{subsec:warm_start}}
We compare the training procedure of the following two plans:

\begin{itemize}
\item Cold start: We ignore the previous network of $\mathcal{N},$ and train a new network from scratch for the new set $\mathcal{N} \cup \mathcal{N}'$.
\item Warm start: We initialize the weights of the new network according to the previously well-trained network of $\mathcal{N}$.
\end{itemize}

Figure~\ref{fig:WarmStart} presents the validation losses under the two training schemes for the two models with the old set $|\mathcal{N}|=20$ and the new set $|\mathcal{N}'|=5$. The experiment shows that, when there is an abundant amount of data for augmented models, warm start training will converge to the optimal solution faster in starting stages as the right panel. More importantly, when there is not enough data, warm start training will not only be faster but also produce a better solution. The second observation can be practically useful, as the retails may carry out small-scale experiments to test if it is profitable to introduce more products to their shops. The experiment tells that the neural networks can be trained with a small amount of new data if most of its parameters are inherited from a well-trained model.

\textbf{Experiment details of Figure~\ref{fig:WarmStart}:} We use MCCM to generate the synthetic data for Figure~\ref{fig:WarmStart}. We first generate the arriving probability $\lambda$ and the transition matrix $\rho$ for the underlying MCCM. We consider $25$ products and so $\lambda \in \mathbb{R}^{26}, \rho \in \mathbb{R}^{26\times 26}$, as product 0 (the no-purchase option) needs to be included in addition. Following this MCCM we generate $200,000$ pieces of samples, with each sample containing an assortment and the corresponding choice probability and the one-hot encoded final choice vector. We call this synthetic dataset D-augment. Then, based on the $\lambda$ and $\rho$ for D-augment, we can do some modifications to classify product $21$ to $25$ into product $0$. We get $\lambda'\in \mathbb{R}^{21}$ such that $\lambda'_i=\lambda_i$ for $i=1,...,20$ and $\lambda'_0=\lambda_0+\sum_{i=21}^{25}\lambda_i$. In similar manners we get $\rho'\in \mathbb{R}^{21\times 21}$, corresponding to a shrunk version of the original MCCM, i.e., a new MCCM with total number of products $n=20$ plus one no-purchase option. For this new model we also generate $200,000$ pieces of samples, and we call this dataset D-shrink.

We first train models on dataset D-shrink, the Gated-Assort-Net has one hidden layer of size $50$, that is a Linear$(20, 50)$ followed by a Linear$(50, 20)$, and the Res-Assort-Net has one residual block, also with one hidden layer of size $50$. We train the Gated-Assort-Net and Res-Assort-Net on $100,000$ training samples. After that, we consider the augmented dataset D-augment, and we train Gated-Assort-Net and Res-Assort-Net with one hidden layer of size $50$. But we consider two methods for training the augmented model. One is the cold start method, in which we randomly initialize the model parameters and train the model on D-augment; the other is the warm start method, in which we initialize parameters at corresponding positions with the models trained on D-shrink and then train the model on D-augment. The results are given in Figure~\ref{fig:WarmStart}, on which we plot how validation loss varies with training epochs. Then left panel is the training on D-augment with training sample size $m=2,000$, and the right panel is the training with sample size $m=100,000$.

\subsection{Appendix for Subsection~\ref{subsec:meta}: Hotel Data Example}
\label{apnx:Hotel}
We apply  Algorithm~\ref{alg:meta} to a public hotel dataset. The data is from \cite{bodea2009data} and it is originally collected from five U.S. properties of a major hotel chain. For each customer, the offered products in assortment are the different room types, such as king room smoking and 2 double beds room non-smoking. We preprocess the data  as in \cite{csimcsek2018expectation},  removing purchases without matched room type in the assortment,  room types with few purchases (less than $10$), and adding an auxiliary no-purchase option (create four no-purchase records for each purchase record). The summary statistics of the dataset is given in Table~\ref{tab:Hotel_data}.
\begin{table}[ht!]
    \centering
    \begin{tabular}{c|c|c|c|c|c}
    \toprule
       &Hotel $1$&Hotel $2$&Hotel$3$&Hotel $4$&Hotel$5$ \\
\midrule
\# Samples&8935&870&6180&1330&1065\\
\# Products&12&7&17&8&7\\
\# Train Samples&6000&600&4000&900&800\\
\# Validate Samples&1000&100&1000&100&100\\
\# Test Samples&1000&100&1000&200&100\\
 \bottomrule
    \end{tabular}
    \caption{Description of Hotel data (\cite{bodea2009data}).}
    \label{tab:Hotel_data}
\end{table}

Table~\ref{tab:Hotel_result} demonstrates the performance of two benchmark methods and the two neural network models. The experiments is repeated over 10 times varying the training and test sets randomly and the averaged results  reported: each
time the train/validation/test sets are independently and randomly separated. In this experiment, we train our Gated-Assort-Net and Res-Assort-Net for each of hotels using only samples for the hotel. Both networks are taking the simplest configuration: Gated-Assort-Net has no middle layer and Res-Assort-Net has one residual block. Compared to the benchmark MCCM-EM algorithm, we observe that our models have better performance on the dataset of Hotel $1$ and Hotel $3$. This is due to the fact that these two datasets have much larger number of samples than the other $3$ hotels. This coincides with our previous observations that,  a moderately large number of samples is necessary to fully show the advantage of a neural network model. When the number of samples is small, or possibly, the customer choice pattern is simple, the more analytical model of MNL-MLE or MCCM-EM is preferrable. To improve the performances of neural choice model, we can apply Algorithm~\ref{alg:meta}: we generate abundant synthetic data from the well-trained MCCM, which has the lowest losses in validation data, to feed our neural choice model. Specifically, for each hotel in Hotel 2, Hotel 4 and Hotel 5,  we generate $150,000$ samples as training dataset $\Tilde{\mathcal{D}}_{\text{train}}=\{(i_k,\mathcal{S}_k),k=1...,150,000\}$. The offered assortment $\mathcal{S}_k$ is generated by first randomly choosing assortment size $|\mathcal{S}|$ in $\{1,...,|\mathcal{N}|\}$, and then  randomly picking products from $\mathcal{N}$ and form the assortment $S$ and the choice $i_k$ is generated by the corresponding fitted MCCM for each hotel. We re-train a new neural network on the synthetic data  $\Tilde{\mathcal{D}}_{\text{train}}$ and fine-tune it on $\mathcal{D}_{\text{train}}$. We report the performances on Table~\ref{tab:Hotel_result}: by Algorithm~\ref{alg:meta}, using neural choice model as a meta choice model can outperform or at least as good as the benchmarks.

\begin{table}[ht!]
    \centering
    \begin{tabular}{c|cccccc}
    \toprule
  \multirow{2}{*}{Model}      &\multicolumn{5}{c}{Dataset} \\
       &Hotel $1$&Hotel $2$&Hotel$3$&Hotel $4$&Hotel$5$  \\
\midrule
MNL-MLE
&0.775&0.902&0.980&0.661&0.660\\
MCCM-EM
&0.757&\textbf{0.856}&0.945&0.656&0.655\\

\midrule
Gated-Assort-Net & \textbf{0.748} & 0.869 & 0.934 & 0.678 & 0.680\\
Res-Assort-Net & \textbf{0.748} & 0.879 & \textbf{0.928} & 0.688 & 0.679\\
\midrule
Gated-Assort-Net(Meta) & - & \textbf{0.856} & - & \textbf{0.647} & \textbf{0.648}\\
Res-Assort-Net(Meta) & - & 0.876 & - & 0.648 & \textbf{0.648}\\
 \bottomrule
    \end{tabular}
    \caption{Performance on the Hotel data. The Meta means using neural networks as meta choice models by Algorithm~\ref{alg:meta} and the reported number is the out-of-sample CE loss.}
    \label{tab:Hotel_result}
\end{table}

\section{More Extensions}
\subsection{Assortment Effect Processed by Fully Connected Layers}
\label{subsec:assort_eff}

\begin{figure}[ht!]
    \centering
    \begin{subfigure}[b]{0.5 \textwidth}
        \centering
        \includegraphics[width=1\textwidth]{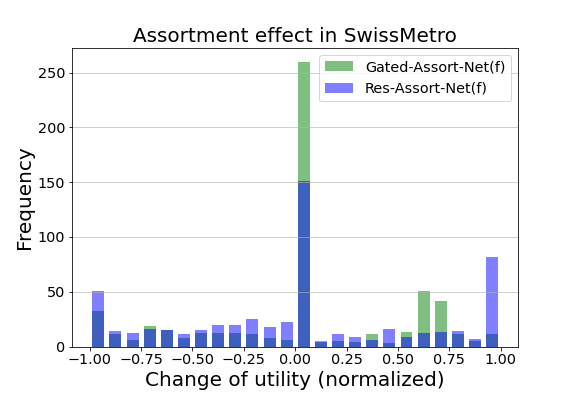}
    \end{subfigure}%
    \begin{subfigure}[b]{0.5 \textwidth}
        \centering
        \includegraphics[width=1\textwidth]{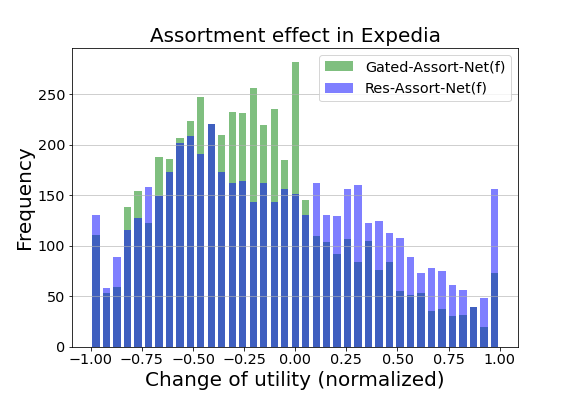}
    \end{subfigure}
    \caption{Effects of fully connected layers in SiwssMetro and Expedia. A larger absolute value means a larger assortment effect from fully connected layers.}
    \label{fig:Eff_layer}
\end{figure}

In practice, we want to understand how the neural networks leverage assortment effect into purchase prediction. Recall the feature-based neural choice models in Figure~\ref{fig:net_feat}, the latent utilities are processed by fully connected layers with assortment information progagated through logic gates in Gated-Assort-Net(f) or directly   as input features in Res-Assort-Net(f).  To interpret such processing, we compute a normalized difference between the input latent utility and output utility of each network. Such a difference reflects the assortment effect in the network.  We denote such difference as $\Delta u$ and plot the histograms of Gated-Assort-Net and Res-Assort-Net in Figure~\ref{fig:Eff_layer}.

From Figure~\ref{fig:Eff_layer}, we find that effects from fully connected layers will vary due to different dataset structures and networks.  For the SwissMetro data set, the majority of changes in normalized utilities $\Delta u$ are located in $0$ for both networks. One plausible reason is that there are only three products and only a small number of samples ($13.5\%$)  do not offer all  the products. Thus, the effect of assortments through the fully connected layer is very small. However, since the Expedia data has more complicated assortment patterns, the $\Delta u$'s do not concentrate at $0$ anymore and have more variance. This also explains why the Res-Assort-Net(f) that has a stronger involvement of the assortment vector, shows an advantage on the Expedia dataset.

Further, by comparing Gated-Assort-Net(f) and Res-Assort-Net(f) in both data sets, there are two interesting observations: (1) There are more $\pm1$'s in Res-Assort-Net(f), which means more dramatic  changes in normalized utilities after passing the fully connected layers; (2) The histograms of Res-Assort-Net(f) are flatter than those of Gated-Assort-Net(f), which means the influences of fully connected layers have more variations. Both of the above observations can be partially explained by the architectures: the Gated-Assort-Net(f) only uses assortments as logical gates but Res-Assort-Net(f) also concatenates it as input features and can squeeze more information to process the input latent utilities, which leads to potentially larger degree of influence with more variations.

\textbf{Experiment details of Figure~\ref{fig:Eff_layer}:}  For each offered product $i\in \mathcal{S}$, we first normalize its input and output utilities of the fully connected layers (from Gated-Assort-Net(f) or Res-Assort-Net(f)) by
$$\tilde{u}_i^{I/O}=\frac{u^{I/O}_i-\min_{j\in \mathcal{S}}u^{I/O}_j}{\max_{j\in \mathcal{S}}u^{I/O}_j-\min_{j\in \mathcal{S}}u^{I/O}_j},$$
where $u^{I}$ denotes input utilities and $u^{O}$ denotes outputs of the fully connected layers. Note that the normalization is within assortment: different assortments may have different scales on utilities and what matters is the relative value.

Then we compute the difference:
$$\Delta u_i\coloneqq \tilde{u}_i^{O}-\tilde{u}^{I}_i$$
for each $i\in \mathcal{S}$ of the fully connected layers. $\Delta u_i=0$ means the normalized utilities are same before and after the processing of fully connected layers. If $\Delta u_i$ is $-1$ or $1$, product $i$ is dramatically influenced: before processing it has the largest utility (smallest utility), i.e. $\tilde{u}^{I}_i=1$ ($\tilde{u}^{I}_i=0$), but after processing it becomes the smallest (largest).

We randomly choose $200$ samples from the test data to plot Figure~\ref{fig:Eff_layer}. The configurations of Gated-Assort-Net(f) and Res-Assort-Net(f) are chosen as the one with the best performance shown in Table~\ref{tab:RD}.

\subsection{Model Calibration}
\label{subsec:calibration}

 Another important aspect of overall performance is the matter of calibration. Ideally, we hope the predicted probability matches the empirical/true probability. For deep learning models, a high accuracy may often come with an overconfidence in predictions (overestimation of the true probabilities) \citep{guo2017calibration}, which is called \textit{miscalibration}. Figure~\ref{fig:calibrate_swiss} is the calibration plot of predicting Metro for the SwissMetro dataset, and we observe that all the neural network models stay mostly below the 45-degree reference line, i.e., suffer from an over-confidence in predicting the probability.
\begin{figure}
    \centering
    \includegraphics[scale=0.4]{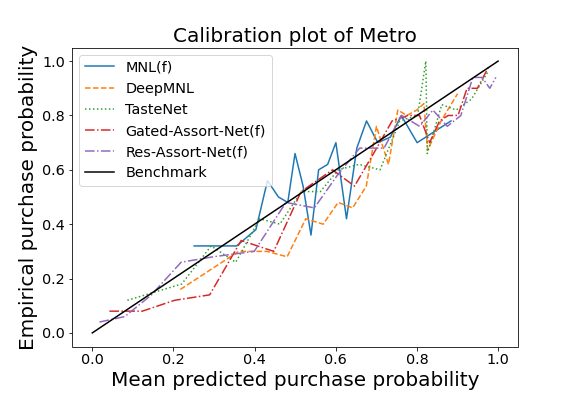}
    \caption{Overconfidence of predicting Metro. The predicted probabilities of the test samples are binned into $20$ bins with the same number of samples. For each bin we compute the mean predicted prob. as x-axis and the mean empirical prob. as y-axis. The benchmark is black solid line where $x=y$.}
    \label{fig:calibrate_swiss}
\end{figure}

 Table~\ref{tab:Calib} summarizes the calibration performances of different models averaged for all choices. In general, both Gated-Assort-Net(f) and Res-Assort-Net(f) have good performances in calibration, which complements the low out-of-sample CE losses shown in Table~\ref{tab:RD} in the sense of reliability.

\begin{table}[ht!]
    \centering
    \begin{tabular}{c|cc}
    \toprule
  \multirow{2}{*}{Model}    &\multicolumn{2}{c}{Datasets}\\
& SwissMetro & Expedia \\
\midrule
MNL-MLE
&4.56&0.79\\
MNL(f)-MLE
&6.22&0.61\\
\midrule
TasteNet+MNL
&5.50&--\\
TasteNet+MNL(f)
&4.66&0.60\\
DeepMNL+MNL
&8.19&--\\
DeepMNL+MNL(f)
&7.26&0.59\\
Random Forest
&7.56&0.41\\
\midrule
Gated-Assort-Net(f)
&4.67&0.39\\
Res-Assort-Net(f)&4.65&0.42\\
 \bottomrule
    \end{tabular}
    \caption{Calibration errors on SwissMetro and Expedia by ACE ($\times10^{-2}$).}
    \label{tab:Calib}
\end{table}

\textbf{Metric calculation for Table~\ref{tab:Calib}:} We use the Adaptive Expected Calibration Error (ACE) as our metric to measure calibration performance \citep{nixon2019measuring}:
$$\text{ACE}=\frac{1}{mB}\sum_{i=1}^n n_i\sum_{b=1}^B|\text{acc}(b,i)-\text{conf}(b,i)|,$$
where $m,B,n$ are the number of samples (assortments), bins and products. Here $n_i$ is the number of assortment samples where product $i$ is included, $\text{acc}(b,i)$, $\text{conf}(b,i)$ are the mean of empirical purchase probabilities (confidence) and mean of predictions (accuracy) in $b$-th bin for product $i$. Further, for each product $i$, the bins are chosen by equal-mass, i.e., each bin will contain the same number of samples.

For Table~\ref{tab:Calib}, we choose the number of bins $B=25$. The configurations of models are chosen as the one with lowest CE losses in Table~\ref{tab:RD}. The training and testing data are same as in Appendix~\ref{sec: real_data_exp}.
\end{document}